\documentclass{article}
\usepackage{ijcai20}

\usepackage{times}
\usepackage{soul}
\usepackage{url}
\usepackage[utf8]{inputenc}
\usepackage[small]{caption}
\usepackage{graphicx}
\usepackage{amsmath}
\usepackage{amsthm}
\usepackage{booktabs}
\usepackage{algorithm}
\usepackage{algorithmic}
\usepackage[colorlinks=true]{hyperref}
\usepackage{color,xcolor}
\hypersetup{
  colorlinks,
  linkcolor={magenta}, 
  citecolor={blue}, 
  urlcolor={blue}, 
}

\usepackage{wrapfig}

\usepackage{subcaption}

\usepackage{algorithm}
\usepackage{algorithmic}
\usepackage{comment}

\usepackage{amsmath,amssymb}
\usepackage{amsbsy}
\usepackage{bm}
\usepackage{amsthm}

\newcommand{\argmin}{\mathop{\mathrm{argmin}}}

\newcommand{\diag}[1]{\mathrm{diag}\left(#1\right)}

\newcommand{\Eqref}[1]{Eq. \eqref{#1}}

\newcommand{\boldy}{{\boldsymbol{y}}}

\newcommand{\boldxi}{{\boldsymbol{\xi}}}

\newcommand{\calF}{{\mathcal F}}
\newcommand{\calG}{{\mathcal G}}

\newcommand{\calT}{{\mathcal T}}

\newcommand{\calX}{{\mathcal X}}
\newcommand{\calY}{{\mathcal Y}}
\newcommand{\calZ}{{\mathcal Z}}

\newcommand{\fhat}{\widehat{f}}

\newcommand{\Jhatell}[1]{J^{\sharp}_{#1}}
\newcommand{\Jhat}{J^{\sharp}}

\newcommand{\Px}{P_{\calX}}
\newcommand{\Py}{P_{\calY}}
\newcommand{\LPi}{L_2} 

\newcommand{\Real}{\mathbb{R}}
\newcommand{\Natural}{\mathbb{N}}

\newcommand{\EE}{\mathrm{E}}
\newcommand{\dd}{\mathrm{d}}
\newcommand{\fstar}{f^{\ast}}

\newcommand{\ftrue}{f^{\mathrm{o}}}
\newcommand{\htrue}{h^{\mathrm{o}}}
\newcommand{\btrue}{b^{\mathrm{o}}}
\newcommand{\Ftrue}{F^{\mathrm{o}}}

\def\I<#1>{\left\langle #1 \right\rangle}
\def\i<#1>{\left\langle #1 \right\rangle}

\newcommand{\ntr}{n} 
\newcommand{\nval}{n_{\mathrm{val}}}

\newcommand{\Dtr}{D_{\mathrm{tr}}}

\newcommand{\Covhatell}[1]{\widehat{\Sigma}^{(#1)}}
\newcommand{\Sigmahat}{\widehat{\Sigma}}

\newcommand{\Fhat}{\hat{F}}
\newcommand{\Whatell}[1]{\hat{W}^{(#1)}}
\newcommand{\bhatell}[1]{\hat{b}^{(#1)}}
\newcommand{\Nhatell}[1]{\hat{N}_{#1}}
\newcommand{\Wtilell}[1]{\widetilde{W}^{(#1)}}
\newcommand{\btilell}[1]{\tilde{b}^{(#1)}}
\newcommand{\Wsharpell}[1]{W^{\sharp(#1)}}
\newcommand{\bsharpell}[1]{b^{\sharp(#1)}}
\newcommand{\Whatdell}[1]{\ddot{W}^{(#1)}}
\newcommand{\bhatdell}[1]{\ddot{b}^{(#1)}}

\newcommand{\Nhatdell}[1]{\hat{N}'_{#1}}
\newcommand{\Nelltheta}{N_\ell^\theta}

\newcommand{\mhatell}{m^{\sharp}_{\ell}}
\newcommand{\mhat}[1]{m^{\sharp}_{#1}}

\newcommand{\fsharp}{f^{\sharp}}

\newcommand{\zetahat}{\hat{\zeta}}
\newcommand{\gammabar}{{\gamma^*}}

\newcommand{\Rnt}{R_{n,t}}

\newcommand{\EEhat}{\widehat{\EE}}
\newcommand{\op}{\mathrm{op}}

\newcommand{\calFhatm}[1]{\hat{\mathcal{F}}_{\boldsymbol{#1}}}
\newcommand{\calGdeltam}[2]{\mathcal{G}_{#1,\boldsymbol{#2}}}
\newcommand{\calGm}[1]{\mathcal{G}_{\boldsymbol{#1}}}

\newcommand{\fhatbeta}{\beta_n}

\newcommand{\Phinr}{\Phi_{n,r}}
\newcommand{\Psinr}{\Psi_{n,r}}

\newcommand{\deltanone}{\delta_{1}}
\newcommand{\deltantwo}{\delta_{2}}

\newcommand{\cZtheta}{\zeta_{\ell,\theta}}

\newcommand{\welle}[1]{{\tilde{\tau}^{(#1)}}}
\newcommand{\qWconst}{c_{\mathrm{scale}}}

\newcommand{\ldkakko}{[\![}
\newcommand{\rdkakko}{]\!]}

\newcommand{\gain}{a}

\newcommand{\Ours}{Spec}

\newcommand{\alphamax}{\alpha_{\mathrm{max}}}

\allowdisplaybreaks
\newtheorem{Theorem}{Theorem}
\newtheorem{Lemma}{Lemma}

\newtheorem{Assumption}{Assumption}

\newtheorem{Proposition}{Proposition}

\newcommand{\Ltwo}{L_2}
\newcommand{\dx}{d_x}

\newcommand{\Tr}{\mathrm{Tr}}
\newcommand{\F}{\mathrm{F}}
\newcommand{\mell}{m_\ell}
\newcommand{\melle}[1]{m_{#1}}
\newcommand{\czero}{c_0}
\newcommand{\cone}{c_1}
\newcommand{\conedelta}{\hat{c}} 

\newcommand{\Well}[1]{W^{(#1)}}
\newcommand{\bell}[1]{b^{(#1)}}

\newcommand{\cdelta}{c_\delta}

\newcommand{\Rb}{R_b}

\newcommand{\Id}{\mathrm{I}}

\newcommand{\logone}{\log_+}

\newcommand{\Rhatinf}{\hat{R}_{\infty}}

\newcommand{\Rbar}{\bar{R}}
\newcommand{\Rbarb}{\bar{R}_b}

\makeatletter
\newcommand{\figcaption}[1]{\def\@captype{figure}\caption{#1}}
\newcommand{\tblcaption}[1]{\def\@captype{table}\caption{#1}}
\makeatother

\title{Spectral Pruning: Compressing Deep Neural Networks \\ via Spectral Analysis
and its Generalization Error\footnote{Copyright \copyright~2020 International Joint Conferences on Artificial Intelligence
(IJCAI). All rights reserved.}}
\author{
Taiji Suzuki$^{1,2,*}$ 
\and
Hiroshi Abe$^{3,\dagger}$ 
\and
Tomoya Murata$^4$
\and
Shingo Horiuchi$^5$
\and
Kotaro Ito$^4$  
\and
Tokuma Wachi$^4$
\and
So Hirai$^5$ 
\and
Masatoshi Yukishima$^4$ 
\And
Tomoaki Nishimura$^{5,\ddagger}$ 
\affiliations
$^1$The University of Tokyo, Japan \and \\
$^2$Center for Advanced Intelligence Project, RIKEN, Japan \and \\ 
$^3$iPride Co., Ltd., Japan \and \\ 
$^4$NTT DATA Mathematical  Systems Inc., Japan \and \\  
$^5$NTT Data Corporation, Japan, \\ 
\emails 
$^{*}$taiji@mist.i.u-tokyo.ac.jp, 
$^\dagger$abe@ipride.co.jp,~$^\ddagger$tomoaki.nishimura.jp@gmail.com
}
 
\begin{document}

\maketitle

\begin{abstract}
Compression techniques for deep neural network models are becoming very important
for the efficient execution of high-performance deep learning systems on edge-computing devices.
The concept of model compression is also important for analyzing the generalization error of deep learning, known as the compression-based error bound.
However, there is still huge gap between a practically effective compression method and its rigorous background of statistical learning theory. To resolve this issue, we develop a new theoretical framework for model compression
and propose a new pruning method called {\it spectral pruning} based on this framework.
We define the ``degrees of freedom'' to quantify the intrinsic dimensionality of a model
by using the eigenvalue distribution of the covariance matrix across the internal nodes
and show that the compression ability is essentially controlled by this quantity.
Moreover, we present a sharp generalization error bound of the compressed model
and characterize the bias--variance tradeoff induced by the compression procedure.
We apply our method to several datasets to justify our theoretical analyses and show the superiority of the the proposed method.
\end{abstract}

\section{Introduction}
Currently, deep learning is the most promising approach adopted by various machine learning applications such as
computer vision, natural language processing, and audio processing.
Along with the rapid development of the deep learning techniques,
its network structure is becoming considerably complicated.
In addition to the model structure, the model size is also becoming larger,
which prevents the implementation of deep neural network models in edge-computing devices
for applications such as smartphone services, autonomous vehicle driving, and drone control.
To overcome this problem,
model compression techniques such as pruning, factorization \cite{denil2013predicting,NIPS2014_5544}, and quantization \cite{
han2015deep}
have been extensively studied in the literature.

Among these techniques, pruning is a typical approach 
that discards redundant nodes, 
e.g., by explicit regularization such as $\ell_1$ and $\ell_2$ penalization during training
\cite{lebedev2016fast,
NIPS2016_6504,he2017channel}.
It has been implemented as ThiNet \cite{iccv2017ThiNet}, Net-Trim \cite{NIPS2017_6910}, NISP \cite{yu2018nisp},
and so on \cite{denil2013predicting}. 
A similar effect can be realized by implicit randomized regularization such as DropConnect \cite{pmlr-v28-wan13}, which randomly removes
connections during the training phase.
However, only few of these techniques (e.g., Net-Trim \cite{NIPS2017_6910}) are supported by statistical learning theory.
In particular, it unclear which type of quantity controls the compression ability.
On the theoretical side,
compression-based generalization analysis is a promising approach for
measuring the redundancy of a network 
\cite{pmlr-v80-arora18b,zhou2018nonvacuous}. 
However, despite their theoretical novelty, the connection of these generalization error analyses to practically useful compression methods is not obvious.

In this paper, we develop a new compression based generalization error bound and 
propose a new simple pruning method that is compatible with the generalization error analysis.
Our method aims to minimize the information loss induced by compression; in particular, it minimizes the redundancy among nodes instead of merely looking at the amount of information of each individual node.
It can be executed by simply observing the covariance matrix in the internal layers and is easy to implement.
The proposed method is supported by a comprehensive theoretical analysis.
Notably, the approximation error induced by compression is characterized by the notion of the statistical {\it degrees of freedom} \cite{mallows1973some,FCM:Caponetto+Vito:2007}. 
It represents the intrinsic dimensionality of a model and is determined by the eigenvalues of the covariance matrix between each node in each layer.
Usually, we observe that the eigenvalue rapidly decreases 
(Fig. \ref{fig:EigenvaluesMNISTandCIFAR}) for several reasons such as explicit regularization (Dropout \cite{NIPS2013_4882},  
weight decay \cite{krogh1992simple}), and
implicit regularization \cite{pmlr-v48-hardt16,gunasekar2018implicit}, which 
means that the amount of important information processed in each layer is not large.
In particular, the rapid decay in eigenvalues leads to a low number of degrees of freedom.
Then, we can effectively compress a trained network into a smaller one that has fewer parameters than the original.
Behind the theory, there is essentially a connection to the {\it random feature technique} for kernel methods \cite{bach2017equivalence}.
Compression error analysis is directly connected to generalization error analysis.
The derived bound is actually much tighter than the naive VC-theory bound on the uncompressed network
\cite{bartlett2017nearly} and even tighter than recent compression-based bounds \cite{pmlr-v80-arora18b}.
Further, there is a tradeoff between the bias and the variance,
where the bias is induced by the network compression and
the variance is induced by the variation in the training data.
In addition, we show the superiority of our method and experimentally verify our theory with extensive numerical experiments.
Our contributions are summarized as follows:

\begin{itemize}
\item We give a theoretical compression bound which is compatible with a practically useful pruning method, 
and 
propose a new simple pruning method called {\it spectral pruning} for compressing deep neural networks.
\item We characterize the model compression ability by utilizing the notion of the degrees of freedom, which
represents the intrinsic dimensionality of the model.
We also give a generalization error bound when a trained network is compressed by our method
and show that the bias--variance tradeoff induced by model compression appears. The obtained bound is fairly tight compared with existing compression-based bounds and much tighter than the naive VC-dimension bound.
\end{itemize}



\section{Model Compression Problem and its Algorithm}

Suppose that the training data $\Dtr = \{(x_i,y_i)\}_{i=1}^{\ntr}$ are observed, where
$x_i \in \Real^{\dx}$ is an input and $y_i$ is an output that could be a real number ($y_i \in \Real$), a binary label ($y_i \in \{\pm 1\}$), and so on.
The training data are independently and identically
distributed.
To train the appropriate relationship between $x$ and $y$,
we construct a deep neural network model as
$$
f(x) =
(\Well{L} \eta( \cdot) + \bell{L}) \circ \dots
\circ (\Well{1} x + \bell{1}),
$$
where $\Well{\ell} \in \Real^{\melle{\ell + 1} \times \melle{\ell}}$, $\bell{\ell} \in \Real^{\melle{\ell + 1}}$ ($\ell = 1,\dots,L$), and $\eta: \Real \to \Real$
is an activation function (here, the activation function is applied in an element-wise manner; for a vector $x \in \Real^d$, $\eta(x)
=(\eta(x_1),\dots,\eta(x_d))^\top$).
Here, $\mell$ is the width of the $\ell$-th layer such that $m_{L+1} = 1$ (output) and $m_1 = \dx$ (input).
Let $\fhat$ be a trained network obtained from the training data $\Dtr = \{(x_i,y_i)\}_{i=1}^{\ntr}$ where 
its parameters are denoted by $(\Whatell{\ell},\bhatell{\ell})_{\ell=1}^L$,
i.e., $\fhat(x) = (\Whatell{L}\eta(\cdot) + \bhatell{L}) \circ \dots \circ (\Whatell{1} x + \bhatell{1})$.
The input to the $\ell$-th layer (after activation) is denoted by 
$\phi^{(\ell)}(x) =\eta \circ (\Whatell{\ell-1}\eta(\cdot) + \bhatell{\ell-1}) \circ \dots \circ (\Whatell{1} x + \bhatell{1}).$
We do not specify how to train the network $\fhat$, and
the following argument can be applied to any learning method
such as the empirical risk minimizer, the Bayes estimator, or another estimator.
We want to compress the trained network $\fhat$
to another smaller one $\fsharp$ having widths $(\mhatell)_{\ell=1}^L$
with keeping the test accuracy as high as possible.



To compress the trained network $\fhat$ to a smaller one $\fsharp$,
we propose a simple strategy called {\it spectral pruning}.
The main idea of the method is to find the {\it most informative subset} of the nodes.
The amount of information of the subset is measured by how well the selected nodes can explain the other nodes in the layer and recover the output to the next layer.
For example, if some nodes are heavily correlated with each other,
then only one of them will be selected by our method.
The information redundancy can be computed by
a covariance matrix between nodes and a simple regression problem.
We do not need to solve a specific nonlinear optimization problem
unlike the methods in \cite{lebedev2016fast,
NIPS2016_6504,NIPS2017_6910}.



\subsection{Algorithm Description}
Our method basically simultaneously minimizes the {\it input information loss} and {\it output information loss}, which will be defined as follows.

\noindent {\bf (i) Input information loss.} 
First, we explain the input information loss.
Denote $\phi(x) = \phi^{(\ell)}(x)$ for simplicity,  
and let $\phi_J(x) = (\phi_j(x))_{j \in J} \in \Real^{\mhatell}$ be a subvector of $\phi(x)$ corresponding to an index set $J
\in [\mell]^{\mhatell}$,
where $[m] := \{1,\dots,m\}$ (here, duplication of the index is allowed).
The basic strategy is to solve the following optimization problem so that we can recover $\phi(x)$ from $\phi_J(x)$ as accurately as possible:
\begin{align}
& \hat{A}_J  := (\hat{A}^{(\ell)}_J=) \argmin_{A \in \Real^{\mell \times |J|}} \EEhat[\|\phi - A \phi_J\|^2] + \|A\|_\tau^2, 
\label{eq:AJinputaware}
\end{align}
where $\EEhat[\cdot]$ is the expectation with respect to the empirical distribution
($\EEhat[f] = \frac{1}{n}\sum_{i=1}^{\ntr} f(x_i)$)
and $\|A\|_\tau^2 = \Tr[A \Id_\tau A^\top]$
for a regularization parameter  $\tau \in \Real_+^{|J|}$
and $\Id_\tau := \diag{\tau}$ (how to set the regularization parameter will be given in Theorem \ref{thm:WhatellWAphiError}).
The optimal solution $\hat{A}_J$ can be explicitly expressed
by utilizing the
{\it (noncentered) covariance matrix} in the $\ell$-th layer of the trained network $\fhat$,
which is defined as
$
\Sigmahat := \Covhatell{\ell}
= \frac{1}{n} \sum_{i=1}^n \phi(x_i) \phi(x_i)^\top, 
$
defined on the empirical distribution (here, we omit the layer index $\ell$ for notational simplicity).
Let $\Sigmahat_{I,I'} \in \Real^{K \times H}$ for $K,H \in \Natural$ be the submatrix of $\Sigmahat$ for the index sets $I \in [\mell]^{K}$ and
$I' \in [\mell]^{H}$
such that $\Sigmahat_{I,I'} = (\Sigmahat_{i,j})_{i\in I,j\in I'}$.
Let $F = \{1,\dots,\mell\}$ be the full index set.
Then, we can easily see that
$$
\hat{A}_J = \Sigmahat_{F,J}(\Sigmahat_{J,J} + \Id_\tau)^{-1}.
$$
Hence, the full vector $\phi(x)$ can be decoded from $\phi_J(x)$ as
$
\phi(x) \approx \hat{A}_J \phi_J(x) =  \Sigmahat_{F,J}(\Sigmahat_{J,J} + \Id_\tau)^{-1} \phi_J(x).
$
To measure the approximation error, we define
$L^{(\mathrm{A})}_{\tau}(J)  = 
\min_{A \in \Real^{\mell}\times |J|} \EEhat[\| \phi -A \phi_J\|^2] +   \| A \|_\tau^2.$
By substituting the explicit formula $\hat{A}_J$ into the objective, this is reformulated as
\begin{align*}
L^{(\mathrm{A})}_{\tau}(J)  = &
\Tr[
\Sigmahat_{F,F} - \Sigmahat_{F,J}(\Sigmahat_{J,J} +  \Id_\tau)^{-1}
\Sigmahat_{J,F} ].
\end{align*}

\noindent {\bf (ii) Output information loss.} 
Next, we explain the output information loss.
Suppose that we aim to directly approximate the outputs $Z^{(\ell)} \phi$ for a weight matrix $Z^{(\ell)} \in \Real^{m \times \mell}$ with an output size $m \in \Natural$.
A typical situation is 
that $Z^{(\ell)} = \Whatell{\ell}$ 
so that we approximate the output $\Whatell{\ell} \phi$ (the concrete setting of $Z^{(\ell)}$ will be specified in Theorem \ref{thm:WhatellWAphiError}).
Then, we consider the objective
\begin{align*}
L^{(\mathrm{B})}_{\tau} (J)
&   :=  \sum_{j=1}^m  \min\limits_{\alpha \in \Real^{\mell}} 
\left\{\EEhat[( Z_{j,:}^{(\ell)} \phi -\alpha^\top \phi_J)^2] \!+\! \| \alpha^\top \|_\tau^2 \right\} \\
&
 =
\Tr\{ Z^{(\ell)} [
\Sigmahat_{F,F} - \Sigmahat_{F,J}(\Sigmahat_{J,J} + \Id_\tau)^{-1}
\Sigmahat_{J,F} ] Z^{(\ell)\top}
\},
\end{align*}
where
$Z^{(\ell)}_{j,:}$ means the $j$-th raw of the matrix
$Z^{(\ell)}$.
It can be easily checked that the optimal solution $\hat{\alpha}_J$ of the minimum in the definition of $L^{(\mathrm{B})}_{\tau}$ 
is given as $\hat{\alpha}_J =  \hat{A}_J^\top Z_{j,:}^{(\ell) \top}$ for each $j=1,\dots,m$. 

\noindent {\bf (iii) Combination of the input and output information losses.} 
Finally, we combine the input and output information losses and aim to minimize this combination.
To do so, we propose to the use of the convex combination of both criteria for a parameter $0 \leq \theta \leq 1$
and
optimize it with respect to $J$ under a cardinality constraint $|J| = \mhatell$
for a prespecified width $\mhatell$ of the compressed network:
\begin{align}
& \min_J ~~
L^{(\theta)}_{\tau}(J) ~= \theta L^{(\mathrm{A})}_\tau(J) + (1- \theta)L^{(\mathrm{B})}_{\tau}(J) 
\notag \\ &
~~~\mathrm{s.t.}~~~
J \in [\mell]^{\mhat{\ell}}. 
\label{eq:minLjlambda}
\end{align}
We call this method {\bf spectral pruning}.
There are the hyperparameter $\theta$ and regularization parameter $\tau$.
However, we see that it is robust against the choice of hyperparameter in experiments (Sec. \ref{sec:NumExp}).
Let $\Jhatell{\ell}$ be the optimal $J$ that minimizes the objective. 
This optimization problem is NP-hard, but an approximate solution is obtained by
the greedy algorithm since it is reduced to maximization of a monotonic submodular function
\cite{krause2014submodular}. That is, we start from $J = \emptyset$,
sequentially choose an element $j^* \in [\mell]$ that maximally reduces the objective $L^{(\theta)}_{\tau}$,
and add this element $j^*$ to $J$ ($J \leftarrow J \cup \{j^*\}$) until $|J| = \mhatell$ is satisfied.

After we chose an index $\Jhatell{\ell}$ ($\ell = 2,\dots,L$) for each layer, we construct the compressed network $\fsharp$ as
$
\fsharp(x) =
(\Wsharpell{L} \eta( \cdot) + \bsharpell{L}) \circ \dots
\circ (\Wsharpell{1} x + \bsharpell{1})
$,
where $\Wsharpell{\ell} =  \Well{\ell}_{\Jhatell{\ell+1},F} \hat{A}_{\Jhatell{\ell}}^{(\ell)}$ and $\bsharpell{\ell} = \bell{\ell}_{\Jhatell{\ell+1}}$.

An application to a CNN is given in Appendix \ref{sec:convolution_extension}.
The method can be executed in a layer-wise manner, thus it can be applied to networks with complicated structures such as ResNet.

\section{Compression accuracy Analysis and Generalization Error Bound}

In this section, we give a theoretical guarantee of our method.
First, we give the approximation error induced by our pruning procedure in Theorem \ref{thm:WhatellWAphiError}.
Next, we evaluate the generalization error of the compressed network in Theorem \ref{thm:GenErrorCompNet}.
More specifically, we introduce a quantity called the {\it degrees of freedom} \cite{mallows1973some,FCM:Caponetto+Vito:2007,AISTATS:Suzuki:2018,suzuki2020compression} 
that represents the intrinsic dimensionality of the model
and determines the approximation accuracy.

For the theoretical analysis, we define a neural network model with norm constraints on the parameters
$\Well{\ell}$ and $\bell{\ell}$ ($\ell = 1,\dots,L$).
Let $R >0$ and $\Rb>0$ be the upper bounds of the
parameters, and
define the norm-constrained model as
\begin{align*}
 \calF  & :=
\textstyle \{   (\Well{L} \eta( \cdot) + \bell{L}) \circ \dots
\circ (\Well{1} x + \bell{1})  
\mid  \\ &
\textstyle
\max_{j} \|\Well{\ell}_{j,:}\| \leq \frac{ R}{\sqrt{\melle{\ell+1}}},~
\|\bell{\ell}\|_\infty \leq \frac{\Rb}{\sqrt{\melle{\ell + 1}}}~
\},
\end{align*}
where
$\Well{\ell}_{j,:}$ means the $j$-th raw of the matrix
$\Well{\ell}$,
$\|\cdot\|$ is the Euclidean norm,
and $\|\cdot\|_\infty$ is the $\ell_\infty$-norm\footnote{We are implicitly supposing $R,\Rb \simeq 1$ so that $\|\Well{\ell}\|_{\F},
\|\bell{\ell}\|=O(1)$.}.
We make the following assumption for the activation function, which is satisfied
by ReLU 
and leaky ReLU 
\cite{Maas13rectifiernonlinearities}.
\begin{Assumption}
\label{ass:EtaCondition}
We assume that the activation function $\eta$ satisfies
(1) scale invariance: $\eta(a x) = a\eta(x)$ for all $a >0$ and $x\in \Real^d$ and
(2) 1-Lipschitz continuity: $|\eta(x) - \eta(x')| \leq \|x - x'\|$ for all $x,x'\in \Real^d$,
where $d$ is arbitrary.
\end{Assumption}


\subsection{Approximation Error Analysis}

Here, we evaluate the approximation error derived by our pruning procedure.
Let $(\mhatell)_{\ell = 1}^L$ denote the width of each layer of the compressed network $\fsharp$.
We characterize the approximation error between $\fsharp$ and $\fhat$
on the basis of the degrees of freedom with respect to the {\it empirical $L_2$-norm}
$\|g\|_n^2 := \frac{1}{\ntr}\sum_{i=1}^{\ntr} \|g(x_i)\|^2$, which is defined for a vector-valued function $g$.
Recall that the empirical covariance matrix in the $\ell$-th layer is denoted by $\Covhatell{\ell}$.
We define the degrees of freedom as
$$
\textstyle
\Nhatell{\ell}(\lambda)
:= \mathrm{Tr}[ \Covhatell{\ell}( \Covhatell{\ell} + \lambda \Id)^{-1}]
=
\sum\nolimits_{j=1}^{\mell}
\hat{\mu}_j^{(\ell)}/(\hat{\mu}_j^{(\ell)} + \lambda),
$$
where $(\hat{\mu}_j^{(\ell)})_{j=1}^{\mell}$ are the eigenvalues of $\Covhatell{\ell}$ sorted in decreasing order.
Roughly speaking, this quantity quantifies the number of eigenvalues above $\lambda$, and thus it is monotonically decreasing w.r.t. $\lambda$.
The degrees of freedom play an essential role in investigating the predictive accuracy of ridge regression
\cite{mallows1973some,FCM:Caponetto+Vito:2007,bach2017equivalence}. 
To characterize the output information loss, 
we also define the {\it output aware degrees of freedom} with respect to a matrix $Z^{(\ell)}$ as 
$$
\Nhatdell{\ell}(\lambda;Z^{(\ell)}) := \Tr[Z^{(\ell)} \Covhatell{\ell}( \Covhatell{\ell} + \lambda \Id)^{-1}Z^{(\ell)\top}].
$$
This quantity measures 
the intrinsic dimensionality of the output from the $\ell$-th layer for a weight matrix $Z^{(\ell)}$.
If the covariance $\Covhatell{\ell}$ and 
the matrix $Z^{(\ell)}$ are near low rank,
$\Nhatdell{\ell}(\lambda;Z^{(\ell)})$ becomes much smaller than $\Nhatell{\ell}(\lambda)$. 
Finally, we define $\Nelltheta(\lambda) := \theta \hat{N}_\ell(\lambda) + (1-\theta) \Nhatdell{\ell}(\lambda;Z^{(\ell)})$.

To evaluate the approximation error induced by compression, we define $\lambda_\ell > 0$ as 
\begin{align}
\textstyle
\lambda_\ell = \inf\{ \lambda \geq 0 \mid \mhatell \geq 5 \Nhatell{\ell}(\lambda) \log(80 \Nhatell{\ell}(\lambda)) \}.
\label{eq:lambdaelldef}
\end{align}
Conversely, we may determine $\mhatell$ from $\lambda_\ell$ to obtain the theorems we will mention below.
Along with the degrees of freedom, we define the {\it leverage score} $\welle{\ell} \in \Real^{\mell}$
as
$
\welle{\ell}_j :=
 \frac{1}{\Nhatell{\ell}(\lambda_\ell)} [\Covhatell{\ell} (\Covhatell{\ell} + \lambda_\ell \Id)^{-1}]_{j,j}~(j\in [\mell]).
$
Note that $\sum_{j=1}^{\mell} \welle{\ell}_j = 1$ originates from the definition of the degrees of freedom.
The leverage score can be seen as the amount of contribution of node $j \in [\mell]$ to the degrees of freedom.
For simplicity, we assume that $\welle{\ell}_j > 0$ for all $\ell,j$ (otherwise, we just need to neglect such a node with $\welle{\ell}_j=0$).

For the approximation error bound, we consider two situations:
(i) (Backward procedure) spectral pruning is applied from $\ell=L$ to $\ell = 2$ in order,
and for pruning the $\ell$-th layer, we may utilize the selected index $\Jhatell{\ell + 1}$ in the $\ell+1$-th layer
and (ii) (Simultaneous procedure) spectral pruning is simultaneously applied for all $\ell = 2,\dots,L$.
We provide a statement for only the backward  procedure. The simultaneous procedure also achieves a similar bound with some modifications.
The complete statement will be given as Theorem \ref{thm:WhatellWAphiError:supple} in Appendix \ref{sec:CompressionErrorBoundProof}.  

As for 
$Z^{(\ell)}$ for the output information loss, we set
$Z^{(\ell)}_{k,:}= {\scriptstyle \sqrt{\melle{\ell} q_{j_k}^{(\ell)}} (\max_{j'} \|\Whatell{\ell}_{j',:}\|)^{-1}} \Whatell{\ell}_{j_k,:}~(k=1,\dots, \mhat{\ell+1})$ 
where we let $\Jhatell{\ell +1}=\{j_1,\dots,j_{\mhat{\ell+1}}\}$,
and  $q_j^{(\ell)} := \frac{( \welle{\ell+1}_j)^{-1}}{\sum_{j' \in \Jhatell{\ell + 1}} (\welle{\ell+1}_{j'})^{-1}}~(j \in \Jhatell{\ell + 1})$ and $q_j^{(\ell)} = 0~(\text{otherwise})$.
Finally, we set the regularization parameter $\tau$ 
as $\tau \leftarrow 
\mhatell \lambda_\ell \welle{\ell}$. 

\begin{Theorem}[Compression rate via the degrees of freedom]
\label{thm:WhatellWAphiError}

If we solve the optimization problem \eqref{eq:minLjlambda}
with the additional constraint $\sum_{j \in J} (\welle{\ell}_j)^{-1} \leq
\frac{5}{3}\mell \mhatell 
$ for the index set $J$, 
then the optimization problem is feasible, and
the overall approximation error of $\fsharp$
is bounded by
\begin{align}
\|\fhat - \fsharp\|_n
\leq
\sum\nolimits_{\ell = 2}^L
\left( \Rbar^{L-\ell + 1} \sqrt{ {\textstyle \prod_{\ell'=\ell}^L}\zeta_{\ell',\theta} } \right) \sqrt{\lambda_\ell}
\label{eq:L2comressionbound}
\end{align}
for $\Rbar = \sqrt{\conedelta}  R$, where $\conedelta$ is a universal constant, and 
$
\cZtheta  :=
\Nelltheta(\lambda_\ell)  \left( \theta \frac{ \max_{j \in [\melle{\ell+1}]} \|\Whatell{\ell}_{j,:}\|^2}{  \|( \Whatell{\ell})^\top \Id_{q^{(\ell)}}\Whatell{\ell}\|_{\mathrm{op}} } + (1-\theta)\melle{\ell}\right)^{-1}$\footnote{$\|\cdot\|_{\op}$ represents the operator norm of a matrix (the largest absolute singular value).}.
\end{Theorem}
The proof is given in Appendix \ref{sec:CompressionErrorBoundProof}. 
To prove the theorem, we essentially need to use theories of random features in kernel methods \cite{bach2017equivalence,AISTATS:Suzuki:2018}.
The main message from the theorem is that
the approximation error induced by compression is directly controlled by
the degrees of freedom.
Since the degrees of freedom $\Nhatell{\ell}(\lambda)$ are a monotonically
decreasing function with respect to $\lambda$,
they become large as $\lambda$ decreases to $0$.
The behavior of the eigenvalues determines
how rapidly $\Nhatell{\ell}(\lambda)$ increases as $\lambda \to 0$.
We can see that if the eigenvalues $\hat{\mu}_1^{(\ell)} \geq \hat{\mu}_2^{(\ell)} \geq \dots$ rapidly decrease,
then the approximation error $\lambda_\ell$ can be much smaller for a given model size $\mhatell$.
In other words, $\fsharp$ can be much closer to the original network $\fhat$
if there are only a few large eigenvalues.

The quantity $\cZtheta$ characterizes how well the approximation error $\lambda_{\ell'}$ of the lower layers $\ell' \leq \ell$
propagates to the final output. 
We can see that a tradeoff between $\cZtheta$ and $\theta$ appears.
By a simple evaluation,
$\Nelltheta$ in the numerator of $\cZtheta$ 
is bounded by $\mell$; thus, $\theta =1$ gives $\cZtheta \leq 1$.
On the other hand, the term
$\frac{ \max_{j \in [\melle{\ell+1}]} \|\Whatell{\ell}_{j,:}\|^2}{  \|( \Whatell{\ell})^\top \Id_{q^{(\ell)}}\Whatell{\ell}\|_{\mathrm{op}} }$
takes a value between $\melle{\ell + 1}$ and 1; thus,
$\theta = 1$ is not necessarily the best choice to maximize the denominator.
From this consideration, we can see that the value of $\theta$ that best minimizes $\cZtheta$
exists between $0$ and $1$, which supports our numerical result (Fig.~\ref{fig:ThetaVSAcc}).
In any situation, small degrees of freedom give a small $\cZtheta$, leading to a sharper bound.

\subsection{Generalization Error Analysis}

Here, we derive the generalization error bound of the compressed network
with respect to the population risk.
We will see that a bias--variance tradeoff induced by
network compression appears.
As usual, we train a network through the training error
$
\widehat{\Psi}(f) := \frac{1}{\ntr} \sum_{i=1}^{\ntr} \psi(y_i,f(x_i))$,
where $\psi:\Real \times \Real \to \Real$ is a loss function.
Correspondingly, the expected error is denoted by $\Psi(f) := \EE[\psi(Y,f(X))]$, where the expectation is taken with respect to
$(X,Y) \sim P$.
Our aim here is to bound the generalization error $\Psi(\fsharp)$ of the compressed network.
Let the marginal distribution of $X$ be $\Px$ and that of $y$ be $\Py$.
First, we assume the Lipschitz continuity for the loss function $\psi$.
\begin{Assumption}
\label{ass:LipschitzLoss}
The loss function $\psi$ is $\rho$-Lipschitz continuous:
$|\psi(y,f) - \psi(y,f')| \leq \rho |f - f'|~
(\forall y \in \mathrm{supp}(\Py),~~\forall f,f' \in \Real)$.
The support of $\Px$ is bounded: 
$
\|x\| 
\leq D_x~~(\forall x \in \mathrm{supp}(\Px)).
$
\end{Assumption}

For a technical reason, 
we assume the following condition for the spectral pruning algorithm.
\begin{Assumption}
\label{ass:AlgAssump}
We assume that $0 \leq \theta \leq 1$ is appropriately chosen so that $\cZtheta$ in Theorem \ref{thm:WhatellWAphiError} satisfies
$\cZtheta \leq 1$ almost surely, and spectral pruning is solved under the condition
 $\sum_{j \in J} (\welle{\ell}_j)^{-1} \leq
\frac{5}{3}\mell \mhatell 
$ on the index set $J$. 
\end{Assumption}
As for the choice of $\theta$, this assumption is always satisfied at least by the backward procedure.
The condition on the linear constraint on $J$ is merely to ensure the leverage scores are balanced for the chosen index.
Note that the bounds in Theorem \ref{thm:WhatellWAphiError} can be achieved even with this condition.

If $L_\infty$-norm of networks is loosely evaluated, 
the generalization error bound of deep learning can be unrealistically large because there appears $L_\infty$-norm in its evaluation.  
However, we may consider a truncated estimator $[\![ \fhat(x)]\!] := \max\{-M,\min\{M,\fhat(x)\}\}$ for sufficiently large $0 < M \leq \infty$
to moderate the $L_\infty$-norm (if $M = \infty$, this does not affect anything).
Note that the truncation procedure does not affect the classification error for a classification task.
To bound the generalization error, we define $\deltanone$ and $\deltantwo$ 
for 
$
(\mhat{2},\dots,\mhat{L})$ and $(\lambda_2,\dots,\lambda_L)$
satisfying relation \eqref{eq:lambdaelldef} as \footnote{$\logone(x) = \max\{1,\log(x)\}$.}
\begin{align*}
&\textstyle
\deltanone  
= \sum_{\ell = 2}^L
\big( \Rbar^{L-\ell + 1}
\text{\small $\sqrt{\prod_{\ell' = \ell}^L \zeta_{\ell',\theta}} \big)$ } \sqrt{ \lambda_\ell},~~
\\ &
\deltantwo^2 = 
\frac{1}{n} \sum_{\ell=1}^L \mhat{\ell}  \mhat{\ell + 1}
\logone \! \left(
{\textstyle
1 +
\frac{4  \hat{G}\max\{\Rbar,\Rbarb\} }{\Rhatinf}
}
 \right), 
\end{align*}
where
%
$
\Rhatinf :=  
\min\{ \Rbar^L D_x + \sum_{\ell = 1}^L \Rbar^{L - \ell} \Rbarb,M\}, ~
\hat{G} := L \Rbar^{L-1} D_x + \sum_{\ell = 1}^L \Rbar^{L - \ell}
$
for $\Rbar = \sqrt{\conedelta } R$ and $\Rbarb = \sqrt{\conedelta} \Rb$ with the constants $\conedelta$
introduced in Theorem \ref{thm:WhatellWAphiError}.
Let
$
\Rnt :=  \frac{ 1 }{n}
\left(
t
+
\sum\nolimits_{\ell=2}^L \log(\mell)
\right)
$ for $t > 0$.
Then, we obtain the following generalization error bound for the compressed network $\fsharp$.

\begin{Theorem}[Generalization error bound of the compressed network]
\label{thm:GenErrorCompNet}
Suppose that Assumptions \ref{ass:EtaCondition}, \ref{ass:LipschitzLoss}, 
and \ref{ass:AlgAssump} are satisfied.
Then, the spectral pruning method presented in Theorem \ref{thm:WhatellWAphiError}
satisfies the following generalization error bound.
There exists a universal constant $C_1 >0$ such that for any $t >0$, it holds that
\begin{align*}
\Psi([\![\fsharp]\!])
\leq  &
\hat{\Psi}([\![\fhat]\!])
+
\rho \Big\{
\deltanone  + C_1 \Rhatinf (\deltantwo + \deltantwo ^2
 +  \sqrt{\Rnt}
)   \Big\}  \\
\lesssim  & 
 \hat{\Psi}([\![\fhat]\!]) \!+\! \sum\nolimits_{\ell=2}^L \!\! \sqrt{\lambda_\ell}
\!+\!
{\textstyle \sqrt{\frac{\sum_{\ell=1}^L \mhat{\ell+1} \mhat{\ell}}{n}{\textstyle \log_+(\hat{G})}}}, 
\end{align*}
uniformly over all choices of $\boldsymbol{\mhat{}} = (\mhat{1},\dots,\mhat{L})$
with probability $1 - 2 e^{-t}$.

\end{Theorem}

The proof is given in Appendix \ref{sec:ProofOfGeneralizationerrorbound}. 
From this theorem, the generalization error of $\fsharp$ is upper-bounded by the training error of the {\it original} network $\fhat$ (which is usually small)
and an additional term.
By Theorem \ref{thm:WhatellWAphiError}, $\deltanone$ represents the approximation error between $\fhat$ and $\fsharp$;
hence, it can be regarded as a {\it bias}.
The second term $\deltantwo$ is the {\it variance} term
induced by the sample deviation.
It is noted that the variance term $\deltantwo$ only depends on the
{\it size of the compressed network rather than the original network size}.
On the other hand, a naive application of the theorem implies 
$\Psi([\![\fhat]\!]) - \hat{\Psi}([\![\fhat]\!]) \leq
 \mathrm{\tilde{O}}\big(  \sqrt{\frac{1}{n}\sum_{\ell=1}^L \melle{\ell+1} \melle{\ell}} \big)$ for the original network $\fhat$,
which coincides with the VC-dimension based bound \cite{bartlett2017nearly} but is much larger than $\deltantwo$
when $\mhat{\ell} \ll \melle{\ell}$.
Therefore, the variance is significantly reduced by model compression, resulting in a much improved generalization error.
Note that the relation between $\deltanone$ and $\deltantwo$ is a tradeoff
due to the monotonicity of the degrees of freedom.
When $\mhatell$ is large, the bias $\deltanone$ becomes small owing to the monotonicity of the degrees of freedom,
but the variance $\deltantwo(\boldsymbol{\mhat{}})$ will be large. 
Hence, we need to tune the size $(\mhatell)_{\ell=1}^L$ to obtain the best generalization error
by balancing the bias ($\deltanone$) and variance ($\deltantwo$).

The generalization error bound is uniformly valid over the choice of $\boldsymbol{\mhat{}} $
(to ensure this, the term $\Rnt$ appears).
Thus, $\boldsymbol{\mhat{}} $ can be arbitrary and chosen in a data-dependent manner.
This means that the bound is {\it a posteriori},
and the best choice of $\boldsymbol{\mhat{}} $ can depend on the trained network.


\section{Relation to Existing Work}


A seminal work \cite{pmlr-v80-arora18b} showed a generalization error bound based on how the network can be compressed.
Although the theoretical insights provided by their analysis are quite instructive,
the theory does not give a practical compression method.
In fact, a random projection is proposed in the analysis, but
it is not intended for practical use.
The most difference is that their analysis exploits the near low rankness of the weight matrix $\Well{\ell}$,
while ours exploits the near low rankness of the covariance matrix $\Covhatell{\ell}$.
They are not directly comparable; thus, we numerically compare the intrinsic dimensionality of both with a VGG-19 network trained on CIFAR-10.
Table \ref{tab:AroraOursComp} summarizes a comparison of the intrinsic dimensionalities.
For our analysis, we used $\Nhatell{\ell}(\lambda_\ell) \Nhatell{\ell+1}(\lambda_{\ell + 1}) k^2 $ for the intrinsic dimensionality
of the $\ell$-th layer, where $k$ is the kernel size\footnote{We omitted quantities related to the depth $L$ and $\log$ term, but
the intrinsic dimensionality of \cite{pmlr-v80-arora18b} also omits these factors.}.
This is the number of parameters in the $\ell$-th layer for the width $\mhat{\ell} \simeq \Nhatell{\ell}(\lambda_\ell)$ where 
$\lambda_\ell$ was set as $\lambda_\ell = 10^{-3} \times \Tr[\Sigmahat_{(\ell)}]$, which is sufficiently small.
We can see that the quantity based on our degrees of freedom give significantly small values in almost all layers.

\begin{table}
\begin{tabular}{crrr}
\hline
Layer & Original &  \cite{pmlr-v80-arora18b} & Spec Prun \\
\hline
1 & 1,728 & 1,645 & 1,013 \\
4& 147,456& 644,654	& 84,499  \\
6& 589,824 & 3,457,882 & 270,216  \\
9 & 1,179,648 &	36,920 & 50,768 	\\
12 & 2,359,296& 22,735 & 4,583  \\
15 & 2,359,296 & 26,584 & 3,886 
 \\
\hline
\end{tabular}
\caption{Comparison of the intrinsic dimensionality of
our degrees of freedom and existing one. 
They are computed for a VGG-19 network trained on CIFAR-10.}
\label{tab:AroraOursComp}
\end{table}

The PAC-Bayes bound \cite{DBLP:conf/uai/DziugaiteR17,zhou2018nonvacuous} 
is also a promising approach for obtaining the nonvacuous generalization error bound of a compressed network.
However, these studies ``assume'' the existence of effective compression methods and do not provide any specific algorithm.
%
\cite{AISTATS:Suzuki:2018,suzuki2020compression} also pointed out the importance of the degrees of freedom for analyzing the generalization error of deep learning
but did not give a practical algorithm. 

\begin{figure*}[ht]
\begin{center}
\begin{minipage}[t]{0.495\textwidth}
\begin{center}
\subcaptionbox{Eigenvalue distributions in each layer for MNIST and CIFAR-10.
\label{fig:EigenvaluesMNISTandCIFAR}}{
\includegraphics[width=5.5cm]{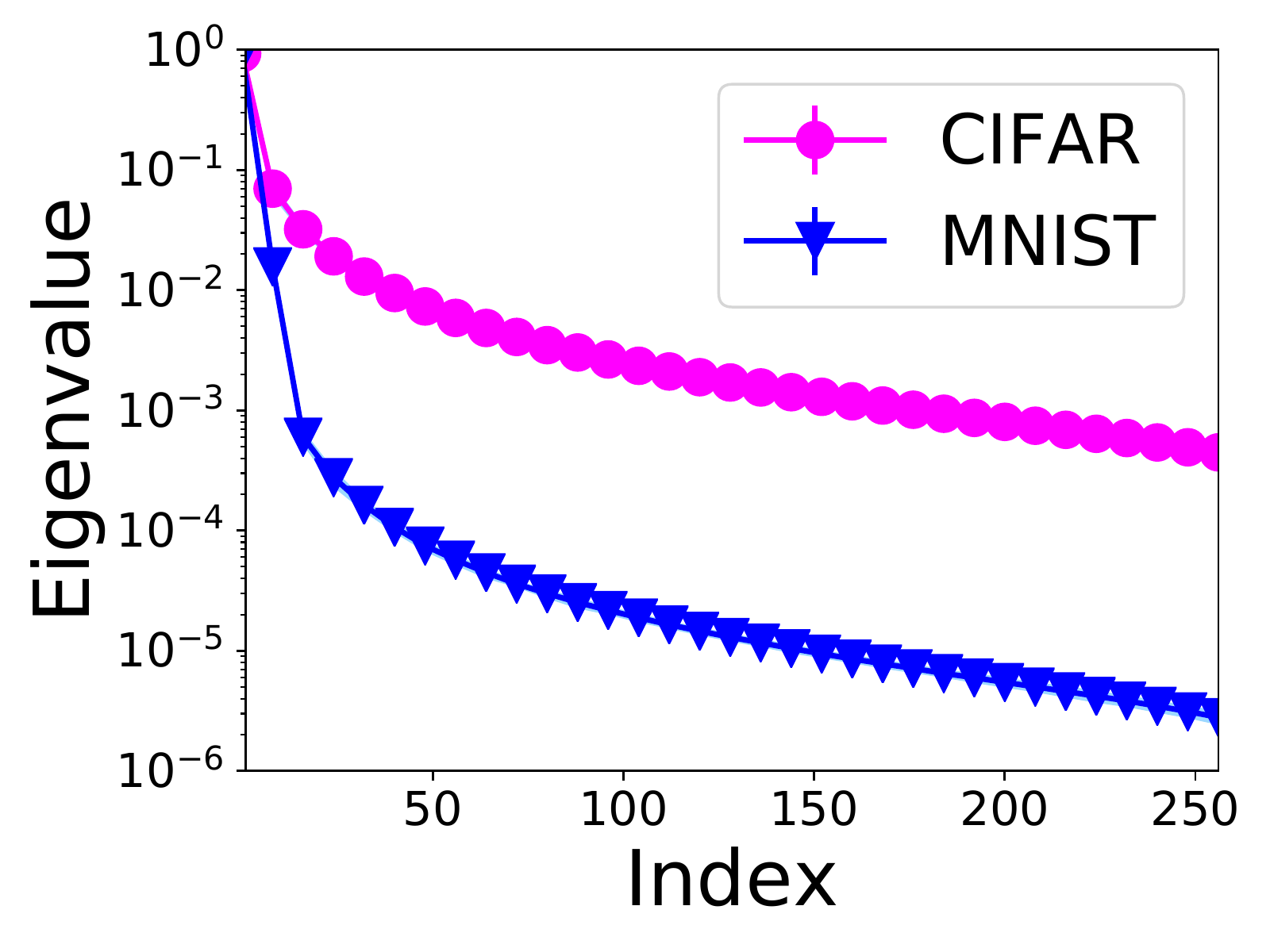}
}
\subcaptionbox{Approximation error with its s.d. versus the width $m^\sharp$
\label{fig:AccuracyInMNISTandCIFAR}
}{
\includegraphics[width=5.5cm]{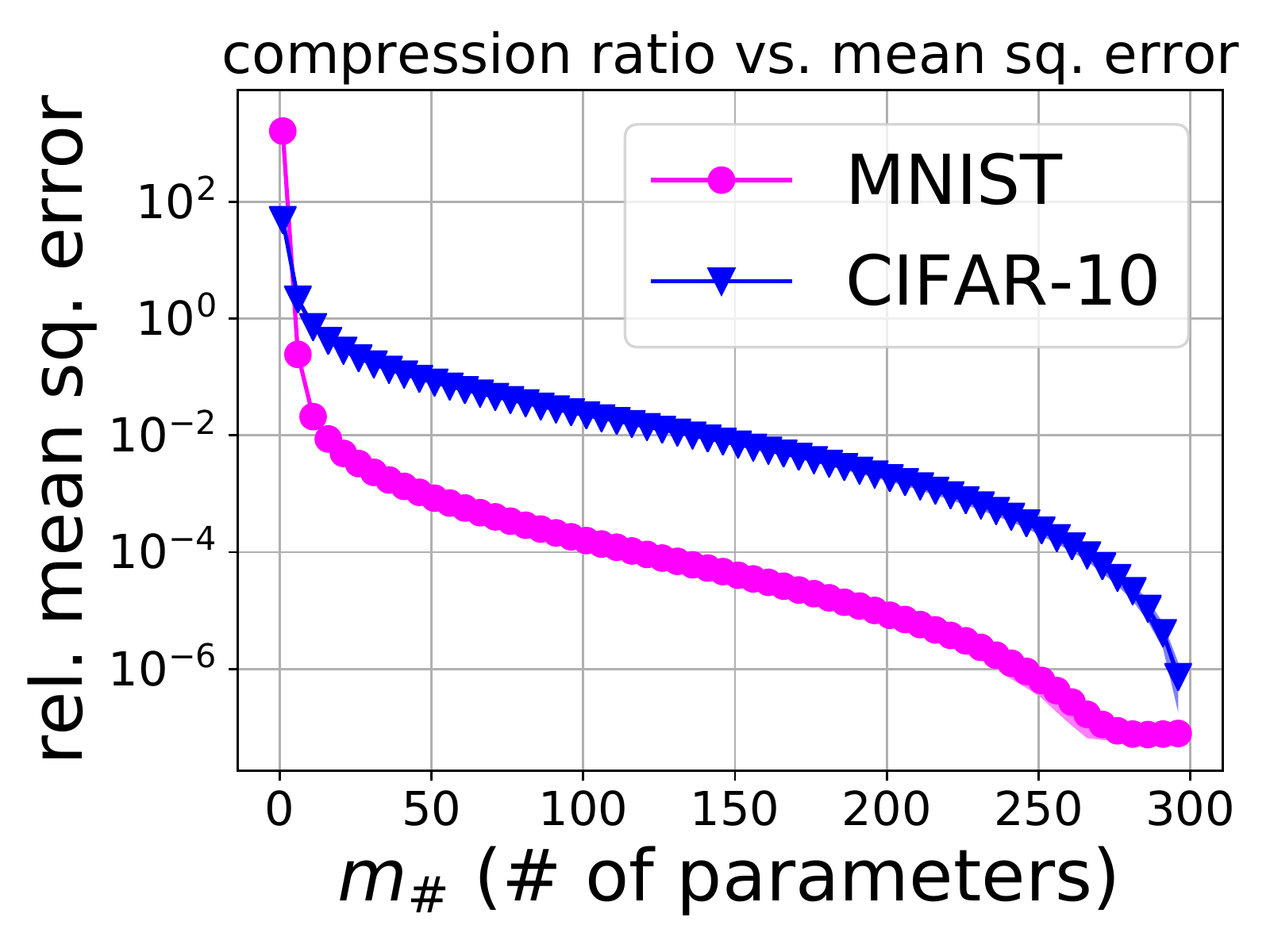} 
}
\caption{Eigenvalue distribution and compression ability
of a fully connected network in MNIST and CIFAR-10.
}
\end{center}
\end{minipage}
~~
\begin{minipage}[t]{.475\textwidth}
\centering
\subcaptionbox{Accuracy versus $\lambda_\ell$ in CIFAR-10 for each compression rate.
\label{fig:LambdaVSAcc}
}{
\includegraphics[width=5.5cm]{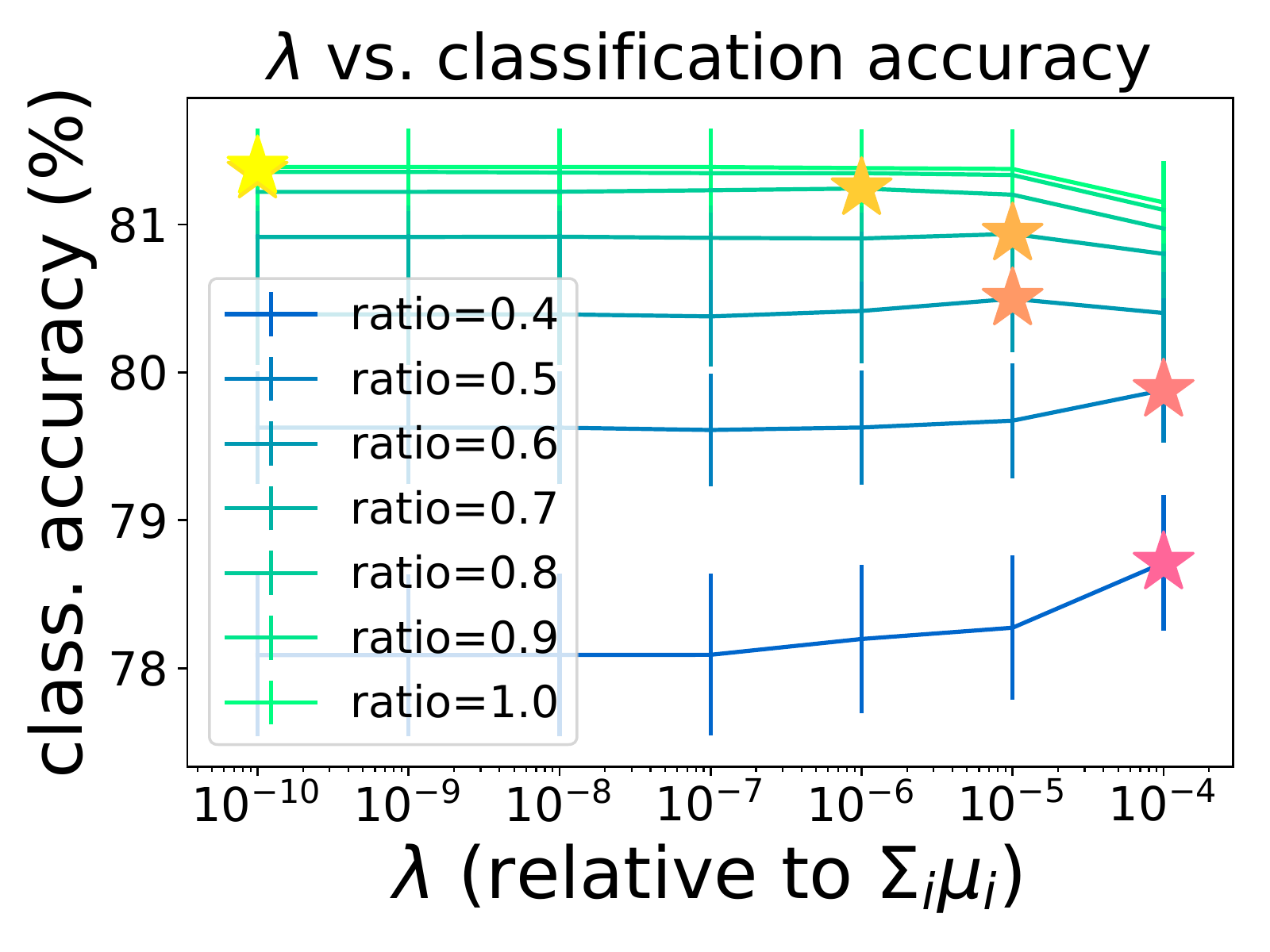} 
}
~
\subcaptionbox{Accuracy versus $\theta$ in CIFAR-10 for each compression rate.
\label{fig:ThetaVSAcc}
}{
\includegraphics[width=5.5cm]{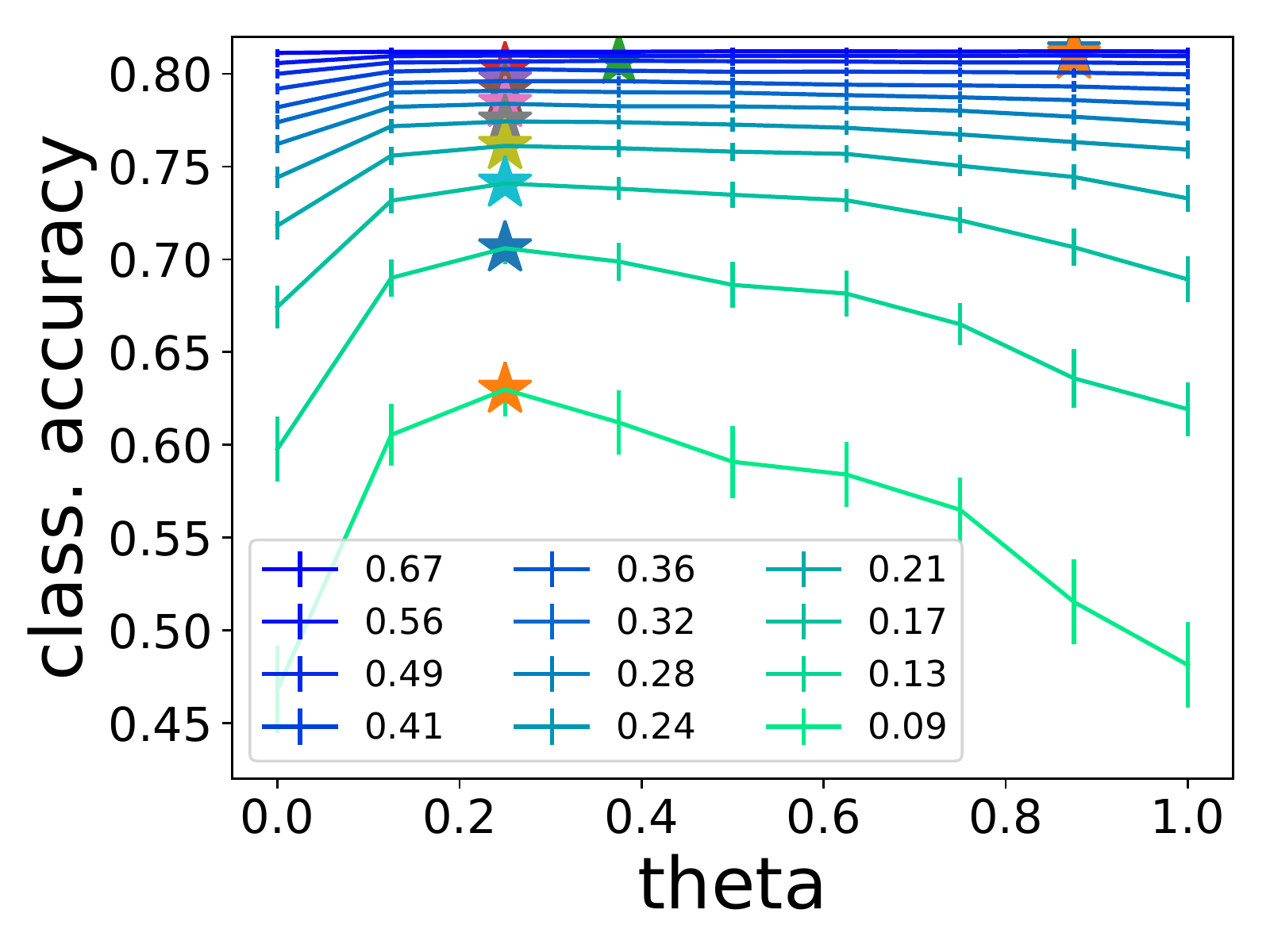} 
}
\caption{
Relation between accuracy and the hyper parameters $\lambda_\ell$ and $\theta$.
The best $\lambda$ and $\theta$ are indicated by the star symbol.
}
\end{minipage}
\end{center}
\end{figure*}


\section{Numerical Experiments}
\label{sec:NumExp}

In this section, we conduct numerical experiments to show the validity of our theory and the effectiveness of the proposed method.


\subsection{Eigenvalue Distribution and Compression Ability}
We show how the rate of decrease in the eigenvalues affects the compression accuracy to justify our theoretical analysis.
We constructed a network (namely, NN3) consisting of three hidden fully connected layers with
widths $(300,1000,300)$ following the settings in \cite{NIPS2017_6910} 
and trained it with 60,000 images in MNIST and 50,000 images in CIFAR-10.
Figure \ref{fig:EigenvaluesMNISTandCIFAR} shows the magnitudes of the eigenvalues
of the 3rd hidden layers of the networks trained for each dataset
(plotted on a semilog scale).
The eigenvalues are sorted in decreasing order, and they are normalized
by division by the maximum eigenvalue.
We see that eigenvalues for MNIST decrease much more rapidly than those for CIFAR-10.
This indicates that MINST is ``easier'' than CIFAR-10 because
the degrees of freedom (an intrinsic dimensionality) of the network trained on
MNIST are relatively smaller than those trained on CIFAR-10.
Figure \ref{fig:AccuracyInMNISTandCIFAR} presents the (relative) compression error $\|\fhat - \fsharp\|_{n}/\|\fhat\|_n$   
versus the width $\mhat{3}$ of the compressed network 
where we compressed only the 3rd layer and $\lambda_3$ was fixed to a constant $10^{-6} \times \Tr[\Sigmahat_{(\ell)}]$ and $\theta = 0.5$. 
It shows a rapid decrease in the compression error for MNIST than CIFAR-10 (about 100 times smaller).
This is because MNIST has faster eigenvalue decay than CIFAR-10.

Figure \ref{fig:LambdaVSAcc} shows the relation between the test classification accuracy and $\lambda_\ell$.
It is plotted for a VGG-13 network trained on CIFAR-10.
We chose the width $\mhat{\ell}$ that gave the best accuracy for each $\lambda_\ell$ under the constraint of the compression rate 
(relative number of parameters).
We see that as the compression rate increases, the best $\lambda_\ell$ goes down.
Our theorem tells that $\lambda_\ell$ is related to the compression error through \eqref{eq:lambdaelldef},
that is, as the width goes up, $\lambda_\ell$ must goes down. This experiment supports the theoretical evaluation.
Figure \ref{fig:ThetaVSAcc} shows the relation between the test classification accuracy and
the hyperparameter $\theta$.
We can see that the best accuracy is achieved around $\theta=0.3$ for all compression rates,
which indicates the superiority of the ``combination'' of input- and output-information loss
and supports our theoretical bound.
For low compression rate, the choice of $\lambda_\ell$ and $\theta$ does not affect the result so much,
which indicates the robustness of the hyper-parameter choice.

\subsection{Compression on ImageNet Dataset}
We applied our method to the ImageNet (ILSVRC2012) dataset \cite{deng2009imagenet}.
We compared our method using the ResNet-50 network \cite{he2016deep} (experiments for VGG-16 network \cite{simonyan2014very} are also shown in Appendix \ref{sec:ImageNetAppendix}).
Our method was compared with the following pruning methods: ThiNet \cite{iccv2017ThiNet},
NISP \cite{yu2018nisp}, and sparse regularization \cite{he2017channel} (which we call Sparse-reg).
As the initial ResNet network, we used two types of networks: ResNet-50-1 and ResNet-50-2.
For training ResNet-50-1, we followed the experimental settings in \cite{iccv2017ThiNet} and \cite{yu2018nisp}.
During training, images were resized as in \cite{iccv2017ThiNet}.
to 256 $\times$ 256; then, a 224 $\times$ 224 random crop was fed into the network.
In the inference stage, we center-cropped the resized images to 224 $\times$ 224.
For training ResNet-50-2, we followed the same settings as in \cite{he2017channel}.
In particular, images were resized such that the shorter side was 256, and a center crop of 224 $\times$ 224 pixels was used for testing.
The augmentation for fine tuning was a 224 $\times$ 224 random crop and its mirror.

We compared ThiNet and NISP for ResNet-50-1 (we call our model for this situation ``Spec-ResA'') and Sparse-reg for ResNet-50-2
(we call our model for this situation ``Spec-ResB'')
for fair comparison.
The size of compressed network $\fsharp$ was determined to be as close to the compared network as possible
(except, for ResNet-50-2, we did not adopt the ``channel sampler'' proposed by \cite{he2017channel} in the first layer of the residual block; hence, our model became slightly larger).
The accuracies are borrowed from the scores presented in each paper, and thus we used different models because the original papers of each model reported for each different model. 
We employed the simultaneous procedure for compression.
After pruning, we carried out fine tuning over 10 epochs,
where the learning rate was $10^{-3}$ for the first four epochs, $10^{-4}$ for the next four epochs, and $10^{-5}$ for the last two epochs.
We employed $\lambda_\ell = 10^{-6}\times \Tr[\Sigmahat_{(\ell)}]$ and $\theta = 0.5$. 

\begin{table}
\label{tab:ImageNet-VGG-results}
\centering
\begin{tabular}{|@{\hspace{0.05cm}}l@{\hspace{0.05cm}}|
@{\hspace{0.1cm}}ll@{\hspace{0.2cm}}r@{\hspace{0.15cm}}r|}
\hline
Model & Top-1 & Top-5 & \# Param. & FLOPs\\
\hline
\hline
ResNet-50-1 
& 
72.89\% & 91.07\% & 25.56M & 7.75G \\
ThiNet-70  
& 72.04 \% & 90.67\% & 16.94M & 4.88G  \\
ThiNet-50  
& 71.01 \% & 90.02\% & 12.38M & 3.41G  \\
NISP-50-A & 72.68\%  &~~~~--- & 18.63M & 5.63G  \\
NISP-50-B &71.99\%  &~~~~--- & 14.57M & 4.32G  \\
\hline
{\bf {\Ours}-ResA} & {\bf 72.99\%} & {\bf 91.56\%} & 12.38M & 3.45G \\
\hline
\hline
ResNet-50-2 & 75.21\% & 92.21\% &  25.56M & 7.75G \\
Sparse-reg  wo/ ft  & ~~~~ ---& 84.2\%  & 19.78M & 5.25G \\
Sparse-reg  w/ ft  &~~~~ ---& 90.8\%  & 19.78M & 5.25G \\
\hline
{\bf {\Ours}-ResB} wo/ ft &
66.12\% & {\bf 86.67\%} & 20.69M &5.25G \\
{\bf {\Ours}-ResB} w/ ft &
74.04\% & {\bf 91.77\%} & 20.69M & 5.25G \\
\hline
\end{tabular}
\caption{Performance comparison of our method and existing ones for ResNet-50 on ImageNet.
``ft'' indicates fine tuning after compression.
}
\label{tab:ResNetComp}
\vspace{-0.2cm}
\end{table}


Table \ref{tab:ResNetComp} summarizes the performance comparison for ResNet-50.
We can see that for both settings, our method outperforms the others for about $1\%$ accuracy.
This is an interesting result because ResNet-50 is already compact \cite{iccv2017ThiNet} and 
thus there is less room to produce better performance. 
Moreover, we remark that all layers were simultaneously trained in our method,
while other methods were trained one layer after another.
Since our method did not adopt the channel sampler proposed by \cite{he2017channel},
our model was a bit larger. However, we could obtain better performance by combining it with our method.

\section{Conclusion}

In this paper,
we proposed a simple pruning algorithm for compressing a network
and gave its approximation and generalization error bounds using the {\it degrees of freedom}.
Unlike the existing compression based generalization error analysis, our analysis is compatible with a practically useful method
and further gives a tighter intrinsic dimensionality bound.
The proposed algorithm is easily implemented and only requires linear algebraic operations. 
The numerical experiments showed that the compression ability is related to the eigenvalue distribution, 
and our algorithm has 
favorable performance compared to existing methods. 

\section*{Acknowledgements}
TS was partially supported by MEXT Kakenhi (18K19793, 18H03201 and 20H00576)
and JST-CREST, Japan.

\bibliographystyle{named_mod}
\bibliography{main,main_2}

\onecolumn

\begin{center}
\bf \LARGE
---Appendix---
\end{center}
\appendix
\appendix 



\section{Extension to convolutional neural network}
\label{sec:convolution_extension}

An extension of our method to convolutional layers 
is a bit tricky.
There are several options, but to perform channel-wise pruning,
we used the following ``covariance matrix'' between channels in the experiments.
Suppose that a channel $k$ receives the input $\phi_{k;u,v}(x)$ where $1 \leq u \leq I_\tau,~1\leq v \leq I_h$ 
indicate the spacial index, then 
``covariance'' between the channels $k$ and $k'$ can be formulated as  
$\Sigmahat_{k,k'} =\frac{1}{\ntr}
\sum_{i=1}^{\ntr} (\frac{1}{I_\tau I_h} \sum_{u,v} \phi_{k;u,v}(x_{i}) \phi_{k';u,v}(x_{i}))$.
As for the covariance between an output channel $k'$ and an input channel $k$ 
(which corresponds to the $(k',k)$-th element of 
$Z^{(\ell)} \Sigmahat_{F,J} = \mathrm{Cov}(Z^{(\ell)} \phi(X),\phi_J(X))$ for the fully connected situation),
it can be calculated as 
$\Sigmahat_{k',k} =\frac{1}{\ntr}
\sum_{i=1}^{\ntr} (\frac{1}{I_\tau I_h} \sum_{u,v}
\frac{1}{I'_{(u,v)}} \sum_{u',v': (u,v) \in \mathrm{Res}(u',v')}
 (Z^{(\ell)} \phi(x_i))_{k';u',v'}(x_{i}) \phi_{k;u,v}(x_{i}))$,
where $\mathrm{Res}(u',v')$ is the receptive field of the location $u',v'$
in the output channel $k'$, and $I'_{(u,v)}$ are the number of locations $(u',v')$
that contain $(u,v)$ in their receptive fields.

\section{Proof of Theorem \ref{thm:WhatellWAphiError}}

\label{sec:CompressionErrorBoundProof}

The output of its internal layer (before activation) is denoted by
$\Fhat_{\ell}(x) = (\Whatell{\ell}\eta(\cdot) + \bhatell{\ell}) \circ \dots \circ (\Whatell{1} x + \bhatell{1}).$
We denote the set of row vectors of $Z^{(\ell)}$ by 
$\mathcal{Z}_\ell$, i.e., $\mathcal{Z}_\ell =\{Z^{(\ell)\top}_{1,:},\dots, Z^{(\ell)\top}_{m,:} \}$. 
Conversely, we may define $Z^{(\ell)}$ by specifying $\calZ_\ell$.

Here, we restate Theorem \ref{thm:WhatellWAphiError} in a complete form that contains both of backward procedure and simultaneous procedure.
\begin{Theorem}[Restated]
\label{thm:WhatellWAphiError:supple}

Assume that the regularization parameter $\tau$ in the pruning procedure \eqref{eq:minLjlambda} 
is defined by the leverage score $\tau \leftarrow \tau^{(\ell)} := \mhatell \lambda_\ell \welle{\ell}$. 

{(i) Backward-procedure:} 
Let $\mathcal{Z}_\ell$ for the output information loss be 
$\mathcal{Z}_\ell = \left\{ \frac{ \sqrt{\melle{\ell} q_j^{(\ell)}}}{\max_{j'} \|\Whatell{\ell}_{j',:}\|}  \Whatell{\ell}_{j,:} \mid j \in \Jhatell{\ell +1}\right\}$
where $q_j^{(\ell)} := \frac{( \welle{\ell+1}_j)^{-1}}{
\sum_{j' \in \Jhatell{\ell + 1}} (\welle{\ell+1}_{j'})^{-1}}~(j \in \Jhatell{\ell + 1})$ and $q_j^{(\ell)} = 0~(\text{otherwise})$,
and define 
$
\cZtheta  = 
\Nelltheta(\lambda_\ell)  \left( \theta \frac{ \max_{j \in [\melle{\ell+1}]} \|\Whatell{\ell}_{j,:}\|^2}{  \|( \Whatell{\ell})^\top \Id_{q^{(\ell)}}\Whatell{\ell}\|_{\mathrm{op}} } + (1-\theta)\melle{\ell}\right)^{-1}$\footnote{$\|\cdot\|_{\op}$ represents the operator norm of a matrix (the largest absolute singular value).}.
Then, 
if we solve the optimization problem \eqref{eq:minLjlambda}
with an additional constraint $\sum_{j \in J} (\welle{\ell}_j)^{-1} \leq 
\frac{5}{3}\mell \mhatell 
$ for the index set $J$, 
then the optimization problem is feasible, and
the overall approximation error of $\fsharp$
is bounded by
\begin{align}
\|\fhat - \fsharp\|_n
\leq  
\sum\nolimits_{\ell = 2}^L 
\left( \Rbar^{L-\ell + 1} \sqrt{ {\textstyle \prod_{\ell'=\ell}^L}\zeta_{\ell',\theta} } \right) \sqrt{\lambda_\ell},
\label{eq:L2comressionbound}
\end{align}
for $\Rbar = \sqrt{\conedelta}  R$ where $\conedelta$ is a universal constant.

(ii) Simultaneous-procedure: Suppose that there exists $\qWconst > 0$ such that 
\begin{align}
\|\Whatell{\ell}_{j,:}\|^2 \leq \qWconst R^2 \welle{\ell+1}_{j},
\label{eq:cscaleAssump}
\end{align} 
and we employ $\calZ_\ell = \{\Whatell{\ell}_{j,:}/\|\Whatell{\ell}_{j,:}\| \mid j \in [\melle{\ell+1}]\}$ for the output aware objective.
Then, we have the same bound as  \eqref{eq:L2comressionbound} for $q_j^{(\ell)} = (\welle{\ell+1}_{j})^{-1}/\sum_{j' \in [\melle{\ell + 1}]} (\welle{\ell+1}_{j'})^{-1}~(\forall j \in [\melle{\ell + 1}])$ and 
\begin{align}
\zeta_{\ell, \theta} = 
\qWconst \Nelltheta(\lambda_\ell)  \left( \theta {\textstyle \frac{\mhat{\ell+1} \max_{j} q_j^{(\ell)}\|\Whatell{\ell}_{j,:}\|^2 }{\|(\Whatell{\ell})^\top \Id_{q^{(\ell)}} \Whatell{\ell}\|_{\op}  }} + (1- \theta) \mhat{\ell+1} \right)^{-1}.
\end{align}
\end{Theorem}

The assumption \eqref{eq:cscaleAssump} is rather strong, but we see that it is always satisfied by $\qWconst = 1$
when $\lambda_\ell = 0$ and by $\qWconst = \Tr[\Sigmahat_{(\ell+1)}]/(\melle{\ell+1}\min_{j} \Sigmahat_{(\ell+1),(j,j)})$ 
when $\lambda_\ell = \infty$. Thus, it is satisfied if the variances of the nodes in the $\ell+1$-th layer is balanced,
which is ensured if we are applying batch normalization.

\subsection{Preparation of lemmas}

To derive the approximation error bound. 
we utilize the following proposition that was essentially proven by \cite{bach2017equivalence}.
This proposition states the connection between the degrees of freedom and the compression error,
that is, it characterize the sufficient width $\mhat{\ell}$ to obtain a pre-specified compression error $\lambda_\ell$. 
Actually, we will see that the eigenvalue essentially controls this relation through the degrees of freedom.  

\begin{Proposition}
\label{prop:BachFiniteApprox}

There exists a probability measure $q_\ell$ on $\{1,\dots,\mell\}$ 
such that
for any $\delta \in (0,1)$ and $\lambda > 0$,
i.i.d. sample $v_1,\dots,v_m \in \{1,\dots,\mell\}$ from $q_\ell$ satisfies,
with probability $1-\delta$,
that 
\begin{align*}
& \inf_{\beta \in \Real^m} 
\left\{
\left\| \alpha^\top \eta (\Fhat_{\ell -1}(\cdot)) - \sum_{j=1}^m \beta_j q_\ell(v_j)^{-1/2} \eta (\Fhat_{\ell -1}(\cdot))_{v_j} \right\|_n^2 
+ m \lambda \|\beta\|^2 \right\} \\
& 
\leq 4 \lambda \alpha^\top \Sigmahat_\ell (\Sigmahat_\ell + \lambda \Id)^{-1} \alpha,
\end{align*}
for every $\alpha \in \Real^{\mell}$, 
if 
$$
m \geq 5 \Nhatell{\ell}(\lambda) \log(16 \Nhatell{\ell}(\lambda)/\delta).
$$
Moreover, the optimal solution $\hat{\beta}$ satisfies $\|\hat{\beta}\|_2^2 \leq \frac{4 \|\alpha\|^2}{m}$.
\end{Proposition}
\begin{proof}
This is basically a direct consequence from Proposition 1 in \cite{bach2017equivalence}
and its discussions.
The original statement does not include the regularization term  $m \lambda \|\beta\|^2$ in the LHS 
and $\alpha^\top \Sigmahat_\ell (\Sigmahat_\ell + \lambda \Id)^{-1} \alpha$ in the right hand side.
However, by carefully following the proof, the bound including these additional factors
is indeed proven.

The norm bond of $\hat{\beta}$ is guaranteed by the following relation:
$$
m \lambda \|\hat{\beta}\|^2 \leq 4  \lambda \alpha^\top \Sigmahat_\ell (\Sigmahat_\ell + \lambda \Id)^{-1} \alpha
\leq 4  \lambda \|\alpha\|^2.
$$
\end{proof}

Proposition \ref{prop:BachFiniteApprox} indicates the following lemma by the the scale invariance of $\eta$, $\eta(a x) = a \eta(x)~(a >0)$.
\begin{Lemma}
\label{prop:ApproxFinite}
Suppose that 
\begin{equation}
\tau'_j = 
\frac{1}{\Nhatell{\ell}(\lambda)}
\sum_{l=1}^{\mell} 
U_{j,l}^2
\frac{\hat{\mu}_l^{(\ell)}}{\hat{\mu}_l^{(\ell)} + \lambda}
= \frac{1}{\Nhatell{\ell}(\lambda)} [\Covhatell{\ell} (\Covhatell{\ell} + \lambda \Id)^{-1}]_{j,j}
~~(j \in \{1,\dots,\mell\}),
\label{eq:definitionWjp}
\end{equation}
where $U = (U_{j,l})_{j,l}$ is the 
orthogonal matrix that diagonalizes 
$\Covhatell{\ell}$, that is, $\Covhatell{\ell}= U\diag{\hat{\mu}_1^{(\ell)},\dots,\hat{\mu}_{\mell}^{(\ell)}} U^\top$.
For $\lambda >0$, and any 
$1/2 > \delta >0$,
if  
$$
m \geq 5 \Nhatell{\ell}(\lambda) \log(16 \Nhatell{\ell}(\lambda)/\delta),
$$
then there exist $v_1,\dots, v_m \in \{1,\dots,\mell\}$ such that, for every $\alpha \in \Real^{\mell}$,
\begin{align}
& \inf_{\beta \in \Real^m} 
\left\{
\left\| \alpha^\top \eta (\Fhat_{\ell -1}(\cdot)) - \sum_{j=1}^m \beta_j {\tau'_j}^{-1/2} \eta (\Fhat_{\ell -1}(\cdot))_{v_j} \right\|_n^2 
+ m \lambda \|\beta\|^2 \right\} \notag \\
& \leq 4 \lambda \alpha^\top \Covhatell{\ell} (\Covhatell{\ell} + \lambda \Id)^{-1} \alpha,
\label{eq:OptimizationFormulaCompressionBasic}
\end{align}
and 
$$
 \sum_{j=1}^m {\tau_j'}^{-1} \leq (1-2\delta)^{-1} m \times \mell.
$$
\end{Lemma}

\begin{proof}
Suppose that the measure $Q_\ell$ is the counting measure, $Q_\ell(J) = |J|$ for $J \subset \{1,\dots,\mell\}$,
and  $q_\ell$ is a density given by $q_\ell(j) = \tau'_j~(j \in \{1,\dots,\mell\})$
with respect to the base measure $Q_\ell$.
Suppose that $v_1,\dots,v_m \in \{1,\dots,\mell\}$ is an i.i.d. sequence 
distributed from $q_\ell \dd Q_\ell$,
then \cite{bach2017equivalence} showed that this sequence satisfies the assertion given in 
Proposition \ref{prop:BachFiniteApprox}. 
  
Notice that 
$\EE_v[\frac{1}{m} \sum_{j=1}^m q_\ell (v_j)^{-1}] =
 \EE_v[q_\ell (v)^{-1}] = \int_{[\mell]} q_\ell (v)^{-1} q_\ell(v) \dd Q_\ell(v) = \int_{[\mell]} 1 \dd Q_\ell(v) =  \mell$,
thus an i.i.d. sequence $\{v_1,\dots,v_m\}$ satisfies $\frac{1}{m} \sum_{j=1}^m q_\ell (v_j)^{-1} \leq \mell/(1-2 \delta)$ with probability $2\delta$
by the Markov's inequality.
Combining this with Proposition \ref{prop:BachFiniteApprox}, the i.i.d. sequence 
$\{v_1,\dots,v_m\}$ and $\tau'_j = q_\ell (v_j)$ satisfies the condition in the statement with probability 
$1-(\delta + 1- 2\delta ) = \delta >0$. 
This ensures the existence of sequences $\{v_j\}_{j=1}^m$ and $\{\tau'_j\}_{j=1}^m$ that satisfy the assertion.
\end{proof}


%


\subsection{Proof of Theorem \ref{thm:WhatellWAphiError}}

\subsubsection{General fact} 

Since Lemma \ref{prop:ApproxFinite} with $\delta = 1/5$ states that 
if $\mhat{\ell} \geq  5 \Nhatell{\ell}(\lambda_\ell) \log(80 \Nhatell{\ell}(\lambda_\ell))$,
then there exists $J \subset [\mell]^{\mhat{\ell} }$
 such that 
$$
\inf_{\alpha \in \Real^{|J|}}
\|z^\top \phi - \alpha^\top \phi_{J}\|_n^2 + 
\lambda_\ell   |J| \|\alpha\|_{\tau'}^2 \leq 4 \lambda_\ell  
z^\top  \Sigmahat_\ell (\Sigmahat_\ell + \lambda \Id)^{-1} z
~~~(\forall z \in \Real^{\mell}),
$$
and 
\begin{align}
\sum_{j \in \Jhatell{\ell}} (\welle{\ell}_j)^{-1} \leq \frac{5}{3}  \mell \times \mhat{\ell}
\label{eq:winvbound}
\end{align}
is satisfied (here, note that $\tau'$ given in \Eqref{eq:definitionWjp} is equivalent to $\welle{\ell}$).

{\bf Evaluation of $L^{(\mathrm{A})}_{\tau}(J)$:}
By setting $z = e_j~(j=1,\dots,\mell)$ where $e_j$ is an indicator vector which has $1$ at its $j$-th component 
and $0$ in other components, and summing up them for $j=1,\dots,\mell$, it holds that 
$$
L^{(\mathrm{A})}_{\tau}(J) = \inf_{A \in \Real^{\mell \times |J|}} \|\phi - A\phi_J \|_n^2 + \lambda_\ell |J| \|A\|_{\tau'}^2
\leq 4 \lambda_\ell \Nhatell{\ell}(\lambda_\ell),
$$
for the same $J$ as above.
Here, the optimal $A$, which is denoted by $\hat{A}_J$, is given by 
$$
\hat{A}_J = \Sigmahat_{F,J}(\Sigmahat_{J,J} + \Id_\tau)^{-1}.
$$

{\bf Evaluation of $L^{(\mathrm{B})}_{\tau}(J)$:}
By letting $z \in \mathcal{Z}_\ell$ and summing up them,
we also have 
$$
L^{(\mathrm{B})}_{\tau}(J)
= 
\inf_{B \in \Real^{\mhat{\ell + 1} \times |J|}} \| Z^{(\ell)} \phi - B \phi_J \|_n^2 + \lambda_\ell |J| \|B\|_{\tau'}^2
\leq 4 \lambda_\ell \Tr[Z^{(\ell)} \Sigmahat_\ell (\Sigmahat_\ell + \lambda_\ell \Id)^{-1}  Z^{(\ell)\top}].
$$
for the same $J$ as above.
Remind again that the optimal $B$, which is denoted by $\hat{B}_J$, is given by
$$
\hat{B}_J = Z^{(\ell)} \Sigmahat_{F,J}(\Sigmahat_{J,J} + \Id_\tau)^{-1} = Z^{(\ell)} \hat{A}_J.
$$

{\bf Combining the bounds for $L^{(\mathrm{A})}_{\tau}(J)$ and $L^{(\mathrm{B})}_{\tau}(J)$:}
By combining the above evaluation, we have that 
$$
\theta L^{(\mathrm{A})}_{\tau}(\Jhatell{}) + (1-\theta) L^{(\mathrm{B})}_{\tau}(\Jhatell{})
\leq 4 \lambda_\ell \{\theta \Nhatell{\ell}(\lambda_\ell)+
(1-\theta) \Tr[Z^{(\ell)} \Covhatell{\ell} (\Covhatell{\ell} + \lambda_\ell \Id)^{-1}  Z^{(\ell)\top}]\},
$$
where $\Jhatell{\ell}$ is the minimizer of $\theta L^{(\mathrm{A})}_{\tau}(J) + (1-\theta) L^{(\mathrm{B})}_{\tau}(J)$ with respect to $J$.

From now on, we let $\tau = \lambda_\ell \mhat{\ell} \tau'
(=\lambda_\ell |\Jhatell{\ell}| \tau' )$ as defined in the main text.

\subsubsection{ (i) Backward-procedure} 

From now on, we give the bound corresponding to the backward-procedure.
The proof consists of three parts: (i) evaluation of the compression error in each layer, (ii) evaluation of the norm of the weight matrix for the compressed network, and (iii) overall compression error of whole layer.
In (i), we use Lemma \ref{prop:BachFiniteApprox} to evaluate the compression error based on the eigenvalue distribution of the covariance matrix.
In (ii), we again use Lemma \ref{prop:BachFiniteApprox} to bound the norm of the compressed network.
This is important to evaluate the overall compression error because the norm controls how the compression error in each layer propagates to the final output. In (iii), we combine the results in (i) and (ii) to obtain the overall compression error.

First, note that, for the choice of 
$\calZ_\ell = \left\{ \frac{ \sqrt{\melle{\ell} q_j^{(\ell)}}}{\max_{j'} \|\Whatell{\ell}_{j',:}\|}  \Whatell{\ell}_{j,:} \mid j \in \Jhatell{\ell + 1} \right\}$,
it holds that 
$$
L^{(\mathrm{B})}_{\tau}(\Jhatell{\ell}) \leq 4 \lambda_\ell \sum_{z \in \calZ_\ell}z^\top  \Sigmahat_\ell (\Sigmahat_\ell + \lambda_\ell \Id)^{-1} z 
\leq 4 \lambda_\ell  \sum_{z \in \calZ_\ell} \|z\|^2 
\leq 4 \lambda_\ell  \melle{\ell} \sum_{j\in \Jhatell{\ell+1}} q_j^{(\ell)} = 4 \lambda_\ell  \melle{\ell}.
$$

{\bf Compression error bound:}
Here, we give the compression error bound of the backward procedure.
For the optimal $\Jhatell{\ell}$, we have that 
\begin{align*}
 \inf_{\alpha \in \Real^{|\Jhatell{\ell}|}} \|\Whatell{\ell}_{j,:} \phi - \alpha^\top \phi_{\Jhatell{}}\|_n^2 + 
 \|\alpha\|_{\tau}^2 & = 
 \Whatell{\ell}_{j,:}
[\Sigmahat_{F,F} - \Sigmahat_{F,\Jhatell{}}(\Sigmahat_{\Jhat,\Jhat} +  \Id_\tau)^{-1} \Sigmahat_{\Jhat,F} ]
(\Whatell{\ell}_{j,:} )^\top \\
& = \Tr\{[\Sigmahat_{F,F} - \Sigmahat_{F,\Jhat}(\Sigmahat_{\Jhat,\Jhat} +  \Id_\tau)^{-1} \Sigmahat_{\Jhat,F} ]
  (\Whatell{\ell}_{j,:})^\top\Whatell{\ell}_{j,:}  \},
\end{align*}
and the optimal $\alpha$ in the left hand side is given by 
$
\Whatell{\ell}_{j,:} \hat{A}_{\Jhatell{\ell}} 
$.
Hence, it holds that 
\begin{align*}
& 
\sum_{j\in \Jhatell{\ell+1}} 
 \inf_{\alpha \in \Real^{|\Jhatell{\ell}|}} \left\{ \|{q_j^{(\ell)}}^{1/2} \Whatell{\ell}_{j,:} \phi - \alpha^\top \phi_{\Jhatell{\ell}}\|_n^2 + 
 \|\alpha\|_{\tau}^2 \right\}
=\sum_{j\in\Jhatell{\ell+1} } {q_j^{(\ell)}}
 \inf_{\alpha \in \Real^{|\Jhatell{\ell}|}} \left\{ \| \Whatell{\ell}_{j,:} \phi - \alpha^\top \phi_{\Jhatell{\ell}}\|_n^2 + 
 \|\alpha\|_{\tau}^2 \right\}
\\
& = 
\sum_{j\in\Jhatell{\ell+1}}  {q_j^{(\ell)}}\Tr\{[\Sigmahat_{F,F} - \Sigmahat_{F,\Jhatell{\ell}}(\Sigmahat_{\Jhatell{\ell},\Jhatell{\ell}} +  \Id_\tau)^{-1} \Sigmahat_{\Jhatell{\ell},F} ]
  (\Whatell{\ell}_{j,:})^\top\Whatell{\ell}_{j,:}  \} \\
&
=
 \Tr\{[\Sigmahat_{F,F} - \Sigmahat_{F,\Jhatell{\ell}}(\Sigmahat_{\Jhatell{\ell},\Jhatell{\ell}} +  \Id_\tau)^{-1} \Sigmahat_{\Jhatell{\ell},F} ]
 ( \Whatell{\ell})^\top \Id_{q^{(\ell)}} \Whatell{\ell}  \} \\
& \leq 
\Tr[ \Sigmahat_{F,F} - \Sigmahat_{F,\Jhatell{\ell}}(\Sigmahat_{\Jhatell{\ell},\Jhatell{\ell}} +  \Id_\tau)^{-1} \Sigmahat_{\Jhatell{\ell},F} ]
 \|( \Whatell{\ell})^\top \Id_{q^{(\ell)}} \Whatell{\ell}\|_{\mathrm{op}} \\
& = L^{(\mathrm{A})}_{\tau}(\Jhatell{\ell})   \|( \Whatell{\ell})^\top  \Id_{q^{(\ell)}} \Whatell{\ell}\|_{\mathrm{op}}
\leq 
L^{(\mathrm{A})}_{\tau}(\Jhatell{\ell}) 
 \frac{ \|( \Whatell{\ell})^\top  \Id_{q^{(\ell)}} \Whatell{\ell}\|_{\mathrm{op}} }{ \max_{j \in [\melle{\ell+1}]} \|\Whatell{\ell}_{j,:}\|^2}
 \max_{j \in [\melle{\ell+1}]} \|\Whatell{\ell}_{j,:}\|^2  \\
&
\leq 
L^{(\mathrm{A})}_{\tau}(\Jhatell{\ell}) 
 \frac{\melle{\ell} \|( \Whatell{\ell})^\top  \Id_{q^{(\ell)}} \Whatell{\ell}\|_{\mathrm{op}} }{ \max_{j \in [\melle{\ell+1}]} \|\Whatell{\ell}_{j,:}\|^2}
\frac{R^2}{\melle{\ell+1}\mell},
\end{align*}
where where we used the assumption $\max_{j} \|\Well{\ell}_{j,:}\| \leq R/\sqrt{\melle{\ell+1}}$.
In the same manner, we also have that 
\begin{align*}
& 
\sum_{j\in \Jhatell{\ell+1} } 
{q_j^{(\ell)}}
 \inf_{\alpha \in \Real^{|\Jhatell{\ell}|}} 
\left\{ \| \Whatell{\ell}_{j,:} \phi - \alpha^\top \phi_{\Jhatell{\ell}}\|_n^2 + 
 \|\alpha\|_{\tau}^2  \right\} \\
& =
\frac{\max_{j'} \|\Whatell{\ell}_{j',:}\|^2 }{\melle{\ell}}
\sum_{j\in {\Jhatell{\ell+1}}} 
 \inf_{\alpha \in \Real^{|\Jhatell{\ell}|}} \left\| { \textstyle \frac{ \sqrt{\melle{\ell} q_j^{(\ell)}}}{\max_{j'} \|\Whatell{\ell}_{j',:}\|} }\Whatell{\ell}_{j,:} \phi - \alpha^\top \phi_{\Jhatell{\ell}} \right\|_n^2 + 
 \|\alpha\|_{\tau}^2 
=
\\
& 
\leq L^{(\mathrm{B})}_{\tau}(\Jhatell{\ell})  \frac{ \max_{j \in [\melle{\ell+1}]} \|\Whatell{\ell}_{j,:}\|^2}{\melle{\ell}}
\leq L^{(\mathrm{B})}_{\tau}(\Jhatell{\ell})  \frac{R^2}{\mell\melle{\ell+1}}.
\end{align*}


These inequalities imply that
\begin{align}
& \sum_{j\in \Jhatell{\ell+1}} 
q_j^{(\ell)}
\left\{
\|\Whatell{\ell}_{j,:} \phi - \Whatell{\ell}_{j,:} \hat{A}_{\Jhat} \phi_{\Jhat}\|_n^2
+  \|\Whatell{\ell}_{j,:} \hat{A}_{\Jhat} \|_\tau^2 \right\}
 \notag \\
&\leq  4 \lambda_\ell 
\frac{
\{\theta \Nhatell{\ell}(\lambda_\ell)+
(1-\theta) \Tr[Z^{(\ell)} \Sigmahat_\ell (\Sigmahat_\ell + \lambda_\ell \Id)^{-1}  Z^{(\ell)\top}]\}} 
{\theta \frac{ \max_{j \in [\melle{\ell+1}]} \|\Whatell{\ell}_{j,:}\|^2}{ \melle{\ell} \|( \Whatell{\ell})^\top \Id_{q^{(\ell)}}\Whatell{\ell}\|_{\mathrm{op}} } + (1-\theta)}
\frac{R^2}{\mell \melle{\ell+1}} \notag \\
&\leq  4 \lambda_\ell 
\frac{
\{\theta \Nhatell{\ell}(\lambda_\ell)+
(1-\theta) \mell \}} 
{\left[\theta   \frac{ \max_{j \in [\melle{\ell+1}]} \|\Whatell{\ell}_{j,:}\|^2}{ \melle{\ell} \|( \Whatell{\ell})^\top \Id_{q^{(\ell)}}\Whatell{\ell}\|_{\mathrm{op}} } + (1-\theta) \right]\mell }
\frac{R^2}{ \melle{\ell+1}} \notag \\
& \leq 
4 \lambda_\ell\zeta_{\ell,\theta}  \frac{R^2 }{\melle{\ell+1}}.
\label{eq:JhatResidualNormBound}
\end{align}

{\bf Norm bound of the coefficients:}
Here, we give an upper bound of the norm of the weight matrices for the compressed network.
From \eqref{eq:JhatResidualNormBound} and the definition that $\tau^{(\ell)} = \lambda_\ell \mhatell \welle{\ell}$, we have that 
$$
\sum_{j\in \Jhatell{\ell+1} } 
q_j^{(\ell)}  \|\Whatell{\ell}_{j,:} \hat{A}_{\Jhat} \|_{\welle{\ell}}^2
\leq \frac{1}{ \lambda_\ell \mhatell } 4 \lambda_\ell\zeta_{\ell,\theta}  \frac{R^2 }{\melle{\ell+1}}
= 4 \zeta_{\ell,\theta}  \frac{R^2 }{\melle{\ell+1}\mhatell}.
$$

Here, by \Eqref{eq:winvbound}, the condition 
$\sum_{j \in \Jhatell{\ell + 1}} (\welle{\ell+1}_j)^{-1} \leq \frac{5}{3}  \melle{\ell+1}  \mhat{\ell+1}$
is feasible, and under this condition, we also have that 
$$
\sum_{j\in \Jhatell{\ell + 1}} 
(\welle{\ell + 1}_j)^{-1}  \|\Whatell{\ell}_{j,:} \hat{A}_{\Jhat} \|_{\welle{\ell}}^2
\leq 4 \zeta_{\ell,\theta}  \frac{R^2 }{\melle{\ell+1}\mhatell} \times \frac{5}{3}  \melle{\ell+1} \mhat{\ell+1}
= \frac{20}{3}  \zeta_{\ell,\theta}\frac{\mhat{\ell+1}}{\mhat{\ell}} R^2,
$$
where we used the definition $q$
Similarly, the approximation error bound \Eqref{eq:JhatResidualNormBound} can be rewritten as 
\begin{align}
& \sum_{j\in \Jhatell{\ell + 1}} 
(\welle{\ell + 1}_j)^{-1}
\|\Whatell{\ell}_{j,:} \phi - \Whatell{\ell}_{j,:} \hat{A}_{\Jhat} \phi_{\Jhat}\|_n^2
\leq 
 \frac{20}{3}  \lambda_\ell \zeta_{\ell,\theta} \mhat{\ell+1}  R^2.  
\label{eq:lthlayer_L2difference}
\end{align}
For $\ell = L$, the same inequality holds for $\melle{L+1}= \mhat{L+1} = 1$ and $\welle{L+1}_j =1~(j=1)$.

{\bf Overall approximation error bound:}
Given these inequalities, we bound the overall approximation error bound.
Let $\Jhatell{\ell}$ be the optimal index set chosen by Spectral Pruning for the $\ell$-th layer,
and the parameters of compressed network be denoted by
$$
\Wsharpell{\ell} = \Whatell{\ell}_{\Jhatell{\ell+1},[\mell]}\hat{A}_{\Jhatell{\ell}}
\in \Real^{\mhat{\ell+1} \times \mhatell}, ~~\bsharpell{\ell} = \bhatell{\ell}_{\Jhatell{\ell + 1}}
\in \Real^{\mhat{\ell+1}}.
$$
Then, it holds that 
\begin{align*}
\fsharp(x) & = (\Wsharpell{L} \eta(\cdot) + \bsharpell{L}) \circ \dots \circ (\Wsharpell{1} x + \bsharpell{1}).
\end{align*}
Then, due to the scale invariance of $\eta$, we also have
\begin{align*}
\fsharp(x) & = (\Wsharpell{L} \Id_{(\welle{L})^{\frac{1}{2}}} \eta(\cdot) + \bsharpell{L}) \circ
(\Id_{(\welle{L})^{-\frac{1}{2}}} \Wsharpell{L-1} \Id_{(\welle{L-1})^{\frac{1}{2}}} \eta(\cdot) + \Id_{(\welle{L})^{-\frac{1}{2}}} \bsharpell{L-1})
 \dots \circ (\Id_{(\welle{2})^{-\frac{1}{2}}} \Wsharpell{1} x + \Id_{(\welle{2})^{-\frac{1}{2}}} \bsharpell{1}).
\end{align*}
Then, if we define as $\Wtilell{\ell} = \Id_{(\welle{\ell+1})^{-\frac{1}{2}}} \Wsharpell{\ell} \Id_{(\welle{\ell})^{\frac{1}{2}} }$ and $\btilell{\ell} =  \Id_{(\welle{\ell+1})^{-\frac{1}{2}}} \bsharpell{\ell}$, then we also have another representation of $\fsharp$ as 
\begin{align*}
\fsharp(x) & = (\Wtilell{L} \eta(\cdot) + \btilell{L}) \circ \dots \circ (\Wtilell{1} x + \btilell{1}).
\end{align*}
In the same manner, the original trained network $\fhat$ is also rewritten as 
\begin{align*}
\fhat(x) 
& = (\Whatell{L} \eta(\cdot) + \bhatell{L}) \circ \dots \circ (\Whatell{1} x + \bhatell{1}) \\
& = (\Whatell{L} \Id_{(\welle{L})^{\frac{1}{2}} } \eta(\cdot) + \bhatell{L}) \circ 
(\Id_{(\welle{L})^{-\frac{1}{2}}} \Whatell{L-1} \Id_{(\welle{L-1})^{\frac{1}{2}}} \eta(\cdot) + \Id_{(\welle{L})^{-\frac{1}{2}}} \bhatell{L-1})
\circ
\dots \circ (\Id_{(\welle{2})^{-\frac{1}{2}}} \Whatell{1} x + \Id_{(\welle{2})^{-\frac{1}{2}}} \bhatell{1}) \\
& =: (\Whatdell{L} \eta(\cdot) + \bhatdell{L}) \circ \dots \circ (\Whatdell{1} x + \bhatdell{1}),
\end{align*}
where we defined $\Whatdell{\ell} := \Id_{(\welle{\ell+1})^{-\frac{1}{2}}} \Whatell{\ell} \Id_{(\welle{\ell})^{\frac{1}{2}}}$ and $\bhatdell{\ell} := \Id_{(\welle{\ell+1})^{-\frac{1}{2}}} \bhatell{\ell}$.

Then, the difference between $\fsharp$ and $\fhat$ can be decomposed into 
\begin{align*}
& \fsharp(x) - \fhat(x) = (\Wtilell{L} \eta(\cdot) + \btilell{L}) \circ \dots \circ (\Wtilell{1} x + \btilell{1})
- (\Whatdell{L} \eta(\cdot) + \bhatdell{L}) \circ \dots \circ (\Whatdell{1} x + \bhatdell{1}) \\
=& 
\sum_{\ell=2}^L 
\Big\{
(\Wtilell{L} \eta(\cdot) + \btilell{L}) \circ \dots \circ (\Wtilell{\ell+1} \eta(\cdot) + \btilell{\ell+1}) \circ
(\Wtilell{\ell} \eta(\cdot) + \btilell{\ell}) \circ 
(\Whatdell{\ell-1} \eta(\cdot) + \bhatdell{\ell-1}) \circ \dots \circ (\Whatdell{1} x + \bhatdell{1}) \\
& - (\Wtilell{L} \eta(\cdot) + \btilell{L}) \circ \dots \circ (\Wtilell{\ell+1} \eta(\cdot) + \btilell{\ell+1}) \circ
(\Whatdell{\ell} \eta(\cdot) + \bhatdell{\ell}) \circ 
(\Whatdell{\ell-1} \eta(\cdot) + \bhatdell{\ell-1}) \circ \dots \circ (\Whatdell{1} x + \bhatdell{1}) \Big\}.
\end{align*}
We evaluate the $\|\cdot\|_n$-norm of this difference.
First, notice that \Eqref{eq:lthlayer_L2difference} is equivalent to the following inequality: 
\begin{align*}
& \| 
(\Wtilell{\ell} \eta(\cdot) + \btilell{\ell}) \circ 
(\Whatdell{\ell-1} \eta(\cdot) + \bhatdell{\ell-1}) \circ \dots \circ (\Whatdell{1} \cdot + \bhatdell{1})
  \\
&
- (\Whatdell{\ell}_{\Jhatell{\ell+1},[\mell]} \eta(\cdot) + \bhatdell{\ell}) \circ 
(\Whatdell{\ell-1} \eta(\cdot) + \bhatdell{\ell-1}) \circ \dots \circ (\Whatdell{1} \cdot + \bhatdell{1})
\|_n^2
\leq 
 \conedelta \lambda_\ell \zeta_{\ell,\theta}  \mhat{\ell+1}  R^2.
\end{align*}
(We can check that, even for $\ell = 2$, this inequality is correct.)
Next, by evaluating the Lipschitz continuity of the $\ell$-th layer of $\fsharp$ as 
\begin{align*}
\|\Wtilell{\ell} g - \Wtilell{\ell} g' \|_n^2
& = \frac{1}{n} \sum_{i=1}^n \| \Wtilell{\ell} g(x_i)  - \Wtilell{\ell} g'(x_i) \|^2 \\
& = \frac{1}{n} \sum_{i=1}^n (g(x_i)  - g'(x_i))^\top (\Wtilell{\ell})^\top \Wtilell{\ell} (g(x_i)  - g'(x_i)) \\
& \leq \frac{1}{n} \sum_{i=1}^n \|g(x_i)  - g'(x_i)\|^2 \Tr[ (\Wtilell{\ell})^\top \Wtilell{\ell}] \\
& \leq  \conedelta  \zeta_{\ell,\theta} \frac{\mhat{\ell+1}}{\mhat{\ell}} R^2 \|g - g'\|_n^2,
\end{align*}
for $g, g': \Real^d \to \Real^{\mhat{\ell}}$, then it holds that 
\begin{align*}
& \| (\Wtilell{L} \eta(\cdot) + \btilell{L}) \circ \dots \circ (\Wtilell{\ell+1} \eta(\cdot) + \btilell{\ell+1}) \circ
(\Wtilell{\ell} \eta(\cdot) + \btilell{\ell}) \circ 
(\Whatdell{\ell-1} \eta(\cdot) + \bhatdell{\ell-1}) \circ \dots \circ (\Whatdell{1} x + \bhatdell{1}) \\
& - (\Wtilell{L} \eta(\cdot) + \btilell{L}) \circ \dots \circ (\Wtilell{\ell+1} \eta(\cdot) + \btilell{\ell+1}) \circ
(\Whatdell{\ell} \eta(\cdot) + \bhatdell{\ell}) \circ 
(\Whatdell{\ell-1} \eta(\cdot) + \bhatdell{\ell-1}) \circ \dots \circ (\Whatdell{1} x + \bhatdell{1}) \|_n^2 \\
\leq & 
\prod_{\ell' = \ell+1}^L \conedelta  \zeta_{\ell',\theta}\frac{\mhat{\ell'+1}}{\mhat{\ell'}} R^2 \cdot 
\conedelta \lambda_\ell \zeta_{\ell,\theta}  \mhat{\ell+1}  R^2
\leq \lambda_\ell \prod_{\ell' = \ell}^L (\conedelta \zeta_{\ell',\theta} R^2)
= \lambda_\ell \Rbar^{2(L - \ell + 1)} \prod_{\ell'=\ell}^L \zeta_{\ell',\theta}.
\end{align*}
Then, by summing up the square root of this for $\ell=2,\dots,L$, then we have the whole approximation error bound.

\subsubsection{Simultaneous procedure}

Here, we give bounds corresponding to the simultaneous-procedure.
The proof techniques are quite similar to the forward procedure.
However, instead of the $\ell_2$-norm bound derived in the backward-procedure, 
we derive $\ell_\infty$-norm bound for both of the approximation error and the norm bounds.


We let $q_j^{(\ell)} = (\welle{\ell}_j)^{-1}$ for $j=1,\dots,\melle{\ell + 1}$.
As for the input aware quantity $L^{(\mathrm{A})}_{\tau}$,  
for any $j \in [\melle{\ell + 1}]$, it holds that 
\begin{align*}
& 
\sum_{j=1}^{\melle{\ell + 1}} \inf_{\alpha \in \Real^{|\Jhatell{\ell}|}} \left\{ \|{q_j^{(\ell)}}^{1/2} \Whatell{\ell}_{j,:} \phi - \alpha^\top \phi_{\Jhatell{\ell}}\|_n^2 + 
 \|\alpha\|_{\tau}^2 \right\}
=
\sum_{j=1}^{\melle{\ell + 1}} \inf_{\alpha \in \Real^{|\Jhatell{\ell}|}} \left\{ \| \Whatell{\ell}_{j,:} \phi - \alpha^\top \phi_{\Jhatell{\ell}}\|_n^2 + 
 \|\alpha\|_{\tau}^2 \right\}
\\
& = 
\sum_{j=1}^{\melle{\ell + 1}}  {q_j^{(\ell)}} \Whatell{\ell}_{j,:} [\Sigmahat_{F,F} - \Sigmahat_{F,\Jhatell{\ell}}(\Sigmahat_{\Jhatell{\ell},\Jhatell{\ell}} +  \Id_\tau)^{-1} \Sigmahat_{\Jhatell{\ell},F} ](\Whatell{\ell}_{j,:})^\top
\\
& \leq 
\|(\Whatell{\ell})^\top \Id_{q^{(\ell)}} \Whatell{\ell}\|_{\op}  \Tr\{\Sigmahat_{F,F} - \Sigmahat_{F,\Jhatell{\ell}}(\Sigmahat_{\Jhatell{\ell},\Jhatell{\ell}} +  \Id_\tau)^{-1} \Sigmahat_{\Jhatell{\ell},F}  \}\\
& \leq 
\qWconst R^2 \frac{\|(\Whatell{\ell})^\top \Id_{q^{(\ell)}} \Whatell{\ell}\|_{\op}  }{\max_{j} q_j^{(\ell)}\|\Whatell{\ell}_{j,:}\|^2 } L^{(\mathrm{A})}_{\tau}(\Jhatell{\ell}).
\end{align*}
Moreover, as for the output aware quantity $L^{(\mathrm{B})}_{\tau}$,  
we have that 
\begin{align*}
& \sum_{j =1}^{\melle{\ell + 1}} 
{q_j^{(\ell)}}
 \inf_{\alpha \in \Real^{|\Jhatell{\ell}|}} 
\left\{ \| \Whatell{\ell}_{j,:} \phi - \alpha^\top \phi_{\Jhatell{\ell}}\|_n^2 + 
 \|\alpha\|_{\tau}^2  \right\} \\
& =
\sum_{j =1}^{\melle{\ell + 1}} 
{q_j^{(\ell)}}\|\Whatell{\ell}_{j,:}\|^2 
 \inf_{\alpha \in \Real^{|\Jhatell{\ell}|}} \left\| { \textstyle \frac{1}{\|\Whatell{\ell}_{j,:}\|} }\Whatell{\ell}_{j,:} \phi - \alpha^\top \phi_{\Jhatell{\ell}} \right\|_n^2 + 
 \|\alpha\|_{\tau}^2 \\
& 
\leq \qWconst R^2 L^{(\mathrm{B})}_{\tau}(\Jhatell{\ell}). 
\end{align*}

By combining these inequalities, it holds that 
\begin{align*}
\sum_{1 \leq j \leq \melle{\ell + 1}} 
{q_j^{(\ell)}}  \inf_{\alpha \in \Real^{|\Jhatell{\ell}|}} 
\left\{ \| \Whatell{\ell}_{j,:} \phi - \alpha^\top \phi_{\Jhatell{\ell}}\|_n^2 + 
 \|\alpha\|_{\tau}^2  \right\} 
& \leq 
\frac{[\theta  L^{(\mathrm{A})}_{\tau}(\Jhatell{\ell}) + (1 - \theta)  L^{(\mathrm{B})}_{\tau}(\Jhatell{\ell})] }{\theta
{\textstyle  \frac{\max_{j} q_j^{(\ell)}\|\Whatell{\ell}_{j,:}\|^2 }{\|(\Whatell{\ell})^\top \Id_{q^{(\ell)}} \Whatell{\ell}\|_{\op}  }} + (1- \theta)} \qWconst R^2  \\
& \leq 4 \qWconst \lambda_\ell 
\frac{\theta \Nhatell{\ell}(\lambda_\ell) + (1 - \theta) \Nhatdell{\ell}(\lambda_\ell; N_\ell)}
{\theta {\textstyle  \frac{\max_{j} q_j^{(\ell)}\|\Whatell{\ell}_{j,:}\|^2 }{\|(\Whatell{\ell})^\top \Id_{q^{(\ell)}} \Whatell{\ell}\|_{\op}  }} + (1- \theta)} 
R^2
\end{align*}

Therefore, by the definition of $q_j^{(\ell)}$ and $\tau$, it holds that 
\begin{align}
\sum_{1 \leq j \leq \melle{\ell + 1}}
(\welle{\ell + 1}_j)^{-1}  \|\Whatell{\ell}_{j,:} \hat{A}_{\Jhat} \|_{\welle{\ell}}^2
\leq 
4 \qWconst  
\frac{\theta \Nhatell{\ell}(\lambda_\ell) + (1 - \theta) \Nhatdell{\ell}(\lambda_\ell; N_\ell)}{
\mhat{\ell}\left[ \theta {\textstyle  \frac{\max_{j} q_j^{(\ell)}\|\Whatell{\ell}_{j,:}\|^2 }{\|(\Whatell{\ell})^\top \Id_{q^{(\ell)}} \Whatell{\ell}\|_{\op}  }} + (1- \theta)\right]}
R^2
= 4 \cZtheta \frac{\mhat{\ell + 1}}{\mhat{\ell}} R^2.
\label{eq:NormBoundSimultaneous}
\end{align}
Similarly, the approximation error bound can be evaluated as 
\begin{align}
& \sum_{j \in [\melle{\ell + 1}] } 
(\welle{\ell + 1}_j)^{-1}
\|\Whatell{\ell}_{j,:} \phi - \Whatell{\ell}_{j,:} \hat{A}_{\Jhat} \phi_{\Jhat}\|_n^2
\leq 
4 \qWconst \lambda_\ell\frac{\theta \Nhatell{\ell}(\lambda_\ell) + (1 - \theta) \Nhatdell{\ell}(\lambda_\ell; N_\ell)}
{\theta {\textstyle  \frac{\max_{j} q_j^{(\ell)}\|\Whatell{\ell}_{j,:}\|^2 }{\|(\Whatell{\ell})^\top \Id_{q^{(\ell)}} \Whatell{\ell}\|_{\op}  }} + (1- \theta)} 
 R^2
= 4 \lambda_\ell \cZtheta \mhat{\ell + 1} R^2.
\end{align}
This gives the following equivalent inequality: 
\begin{align*}
& \sum_{j \in [\melle{\ell}]} \| 
(\Wtilell{\ell}_{j,:} \eta(\cdot) + \btilell{\ell}) \circ 
(\Whatdell{\ell-1} \eta(\cdot) + \bhatdell{\ell-1}) \circ \dots \circ (\Whatdell{1} \cdot + \bhatdell{1})
  \\
&
~~~~- (\Whatdell{\ell}_{j,:} \eta(\cdot) + \bhatdell{\ell}) \circ 
(\Whatdell{\ell-1} \eta(\cdot) + \bhatdell{\ell-1}) \circ \dots \circ (\Whatdell{1} \cdot + \bhatdell{1})
\|_n^2 \\
&\leq 
4 \qWconst \lambda_\ell [\theta \Nhatell{\ell}(\lambda_\ell) + (1 - \theta) \Nhatdell{\ell}(\lambda_\ell; N_\ell)] R^2.
\end{align*}
Moreover, the norm bound \eqref{eq:NormBoundSimultaneous} gives the following Lipschitz continuity bound of each layer:
\begin{align*}
\sum_{j \in [\melle{\ell + 1}]}  \|\Wtilell{\ell}_{j,:} g - \Wtilell{\ell}_{j,:} g' \|_n^2 
& = \sum_{j \in [\melle{\ell + 1}]} \frac{1}{n} \sum_{i=1}^n ( \Wtilell{\ell}_{j,:} g(x_i)  - \Wtilell{\ell}_{j,:} g'(x_i) )^2 \\
& = \sum_{j \in [\melle{\ell + 1}]}  \|\Wtilell{\ell}_{j,:}\|^2  
\sum_{j' \in [\mell]} \frac{1}{n} \sum_{i=1}^n  (g_{j'}(x_i)  - g_{j'}'(x_i))^2 \\
& \leq  4 \qWconst \frac{\theta \Nhatell{\ell}(\lambda_\ell) + (1 - \theta) \Nhatdell{\ell}(\lambda_\ell; N_\ell)}
{ \mhat{\ell} \left[ \theta {\textstyle \frac{\max_{j} q_j^{(\ell)}\|\Whatell{\ell}_{j,:}\|^2 }{\|(\Whatell{\ell})^\top \Id_{q^{(\ell)}} \Whatell{\ell}\|_{\op}  }} + (1- \theta) \right]} 
R^2
 \|g  - g'\|_n^2
\end{align*}
for $g, g': \Real^d \to \Real^{\mhat{\ell}}$.
Combining these inequalities, it holds that 
\begin{align*}
& \| (\Wtilell{L} \eta(\cdot) + \btilell{L}) \circ \dots \circ (\Wtilell{\ell+1} \eta(\cdot) + \btilell{\ell+1}) \circ
(\Wtilell{\ell} \eta(\cdot) + \btilell{\ell}) \circ 
(\Whatdell{\ell-1} \eta(\cdot) + \bhatdell{\ell-1}) \circ \dots \circ (\Whatdell{1} x + \bhatdell{1}) \\
& - (\Wtilell{L} \eta(\cdot) + \btilell{L}) \circ \dots \circ (\Wtilell{\ell+1} \eta(\cdot) + \btilell{\ell+1}) \circ
(\Whatdell{\ell} \eta(\cdot) + \bhatdell{\ell}) \circ 
(\Whatdell{\ell-1} \eta(\cdot) + \bhatdell{\ell-1}) \circ \dots \circ (\Whatdell{1} x + \bhatdell{1}) \|_n^2 \\
\leq & 
\lambda_\ell \Rbar^{2(L - \ell + 1)}  
\prod_{\ell' = \ell}^L \qWconst 
\frac{[\theta \Nhatell{\ell'}(\lambda_{\ell'}) + (1 - \theta) \Nhatdell{\ell'}(\lambda_{\ell'}; N_{\ell'})]}
{\theta {\textstyle \frac{\max_{j} q_j^{(\ell)}\|\Whatell{\ell}_{j,:}\|^2 }{\|(\Whatell{\ell})^\top \Id_{q^{(\ell)}} \Whatell{\ell}\|_{\op}  }} + (1- \theta)}
\frac{1}{\prod_{\ell' = \ell+1}^L \mhat{\ell'}}.
\end{align*}
By summing up the square root of this for $\ell=2,\dots,L$, we obtain the assertion.

\section{Proof of Theorem \ref{thm:GenErrorCompNet} (Generalization error bound of the compressed network)}
\label{sec:ProofOfGeneralizationerrorbound}
\subsection{Notations}

For a sequence of the width $\boldsymbol{m'} = (m'_{2},\dots,m'_{L})$, 
let 
\begin{align*}
\calFhatm{m'} 
 := & \{f(x) = (\Well{L} \eta( \cdot) + \bell{L}) \circ \dots  
\circ (\Well{1} x + \bell{1})  \notag \\
& \mid \| \Well{\ell}\|_{\F}^2 
\leq 
\Rbar^2,~
\|\bell{\ell}\|_2
\leq 
\Rbarb,~
\Well{\ell} \in \Real^{m'_{\ell+1} \times m'_{\ell}},~
\bell{\ell} \in \Real^{m'_{\ell + 1}}~
(1 \leq \ell \leq L)\}.
\end{align*}

\begin{Proposition}
\label{supplemma:SupnormBounds}
Under Assumptions \ref{ass:EtaCondition} and \ref{ass:LipschitzLoss},  
the 
$\ell_\infty$-norm of $f \in \calFhatm{m'}$ 
is bounded as 
\begin{align*}
\|f\|_\infty
& 
\leq 
\bar{R}^{L} D_x + \sum_{\ell = 1}^L 
\bar{R}^{L-\ell} \Rbarb.
\end{align*}
\end{Proposition}
The proof is easy to see the Lipschitz continuity of the network with respect to $\|\cdot\|$-norm is bounded by $\|W^{(\ell)}\|_{\F}$.

%


By the scale invariance of the activation function $\eta$, $\calFhatm{\mhat{}}$ can be rewritten as 
\begin{align*}
\calFhatm{\mhat{}}
= & \{f(x) = (\Well{L} \eta( \cdot) + \bell{L}) \circ \dots  
\circ (\Well{1} x + \bell{1})  \notag \\
& \mid \| \Well{\ell}\|_{\F}^2 
\leq 
\frac{\mhat{\ell+1}}{\mhat{\ell}} 
\Rbar^2,~
\|\bell{\ell}\|_2
\leq \sqrt{\mhat{\ell + 1}} \Rbarb,~
\Well{\ell} \in \Real^{\mhat{\ell+1} \times \mhat{\ell}},~
\bell{\ell} \in \Real^{\mhat{\ell + 1}}~
(1 \leq \ell \leq L)\}.
\end{align*}
Hence, from Theorem \ref{thm:WhatellWAphiError} and the argument in Appendix \ref{sec:CompressionErrorBoundProof},
we can see that 
under Assumption \ref{ass:AlgAssump}, it holds that 
$$
\fsharp \in \calFhatm{\mhat{}}.
$$
for both of the backward-procedure and the simultaneous-procedure.
Therefore, the compressed network $\fsharp$ of both procedures with the constraint 
has $\ell_\infty$-bound such as 
$$
\|[\![\fsharp]\!]\|_\infty \leq \Rhatinf.
$$

\subsection{Proof of Theorem \ref{thm:GenErrorCompNet}}
\label{sec:ERMproof}

%

Remember that the $\epsilon$-internal covering number of a (semi)-metric space $(T,d)$ is the minimum cardinality of 
a finite set such that every element in $T$ is in distance $\epsilon$ from the finite set with respect to the metric $d$.
We denote by $N(\epsilon,T,d)$ the $\epsilon$-internal covering number of $(T,d)$.
The covering number of the neural network model $\calFhatm{m'}$ can be evaluated as follows (see for example \cite{AISTATS:Suzuki:2018}):
\begin{Proposition}
\label{prop:FmCover}
The covering number of $\calFhatm{m'}$ is bounded by 
$$
\log N(\epsilon, \calFhatm{m'}, \|\cdot\|_\infty) \leq 
C \frac{\sum_{\ell=1}^L m'_\ell m'_{\ell+1}}{n}  
\log_+ \left(1 + \frac{4 \hat{G}\max\{\Rbar,\Rbarb\} }{\delta} \right)
$$
for a universal constant $C > 0$.
\end{Proposition}

We define 
$$
\calGm{m'} = \{g(x_i,y_i) = \psi(y, f(x)) \mid f \in \calFhatm{m'}\},
$$
for $\boldsymbol{m'} = (m_2,\dots,m_L)$.
Then, its Rademacher complexity can be bounded as follows.

\begin{Lemma}
\label{lemm:RademacherBoundOfG}
Let $(\epsilon_i)_{i=1}^n$ be i.i.d. Rademacher sequence, that is,
$P(\epsilon_i = 1) = P(\epsilon_i = -1) = \frac{1}{2}$. 
There exists a universal constant $C > 0$ such that, for all $\delta > 0$, 
\begin{align*}
\EE \left[\sup_{f \in \calGm{m'} }  \left| \frac{1}{n} \sum_{i=1}^n \epsilon_i g(x_i,y_i)  \right|\right]
& 
\leq C
\rho  \Bigg[ \Rhatinf   \sqrt{\frac{\sum_{\ell=1}^L m'_\ell m'_{\ell+1}}{n}  
\log_+ \left(1 + \frac{4 \hat{G}\max\{\Rbar,\Rbarb\} }{\Rhatinf} \right)} \\
&~~
\vee \Rhatinf
\frac{\sum_{\ell=1}^L m'_\ell m'_{\ell+1}}{n}  
\log_+ \left(1 + \frac{4 \hat{G}\max\{\Rbar,\Rbarb\} }{\delta} \right) \Bigg],
\end{align*}
where the expectation is taken with respect to $\epsilon_i,x_i,y_i$.
\end{Lemma}

\begin{proof}

Since $\psi$ is $\rho$-Lipschitz continuous, 
the contraction inequality Theorem 4.12 of \cite{Book:Ledoux+Talagrand:1991} gives an upper bound of the RHS as 
\begin{align*}
& \EE \left[\sup_{g \in \calGm{m'}}   \left| \frac{1}{n} \sum_{i=1}^n \epsilon_i g(x_i,y_i) \right| \right] 
\leq
2 \rho 
\EE \left[\sup_{g \in \calFhatm{m'}}  \left| \frac{1}{n} \sum_{i=1}^n \epsilon_i f(x_i)  \right|\right].
\end{align*}
We further bound the RHS. 
By Theorem 3.1 in \cite{gine2006concentration} or Lemma 2.3 of \cite{IEEEIT:Mendelson:2002} with the covering number bound 
(Proposition \ref{prop:FmCover}), 
there exists a universal constant $C'$ such that
\begin{align*}
\EE \left[\sup_{f \in \calFhatm{m'}}
\left| \frac{1}{n} \sum_{i=1}^n \epsilon_i f(x_i)  \right| \right]  \leq & 
C' 
\Bigg[ \Rhatinf \sqrt{\frac{\sum_{\ell=1}^L m'_\ell m'_{\ell+1}}{n}  
\log_+ \left(1 + \frac{4 \hat{G}\max\{\Rbar,\Rbarb\} }{\Rhatinf} \right)} \\
&
\vee \Rhatinf
\frac{\sum_{\ell=1}^L m'_\ell m'_{\ell+1}}{n}  
\log_+ \left(1 + \frac{4 \hat{G}\max\{\Rbar,\Rbarb\} }{\Rhatinf} \right)
\Bigg].
\end{align*}
This concludes the proof.
\end{proof}

Now we are ready to probe the theorem.
\begin{proof}[Proof of Theorem \ref{thm:GenErrorCompNet}]

Since $\calGm{m'}$ is separable with respect to $\|\cdot\|_\infty$-norm, 
by the standard symmetrization argument, we have that 
$$
P\left\{\sup_{g \in \calGm{m'}}  \left | \frac{1}{n} \sum_{i=1}^n g(x_i,y_i) - \EE_{X,Y}[g]\right| \geq 
2 \EE \left[\sup_{f \in \calGm{m'} }  \left| \frac{1}{n} \sum_{i=1}^n \epsilon_i g(x_i,y_i)  \right|\right] 
+ 3 \Rhatinf \sqrt{\frac{2 t }{n}} \right\} \leq 2 e^{-t}
$$
for all $t > 0$ (see, for example, Theorem 3.4.5 of \cite{GineNickl2015mathematical}).
Taking uniform bound with respect to the choice of 
$\boldsymbol{m'} \in  [m_2] \times [m_3] \times \dots \times [m_L]$,
we have that 
\begin{align}
P\Bigg\{& \sup_{g \in \calGm{m'}}  \left | \frac{1}{n} \sum_{i=1}^n g(x_i,y_i) - \EE_{X,Y}[g]\right| 
\geq 
2 \EE \left[\sup_{f \in \calGm{m'} }  \left| \frac{1}{n} \sum_{i=1}^n \epsilon_i g(x_i,y_i)  \right|\right] 
+ 3 \Rhatinf \sqrt{\frac{2 (t + \sum_{\ell=2}^L \log(\melle{\ell}) )}{n}}  \notag \\
&
\text{for all $\boldsymbol{m'} \in  [m_2] \times [m_3] \times \dots \times [m_L]$ uniformly}
 \Bigg\} \leq 2 e^{-t}.
\label{eq:PunifGmBound}
\end{align}

Now, the generalization error of $\fsharp$ can decomposed into 
\begin{align*}
\Psi([\![\fsharp]\!]) = \underbrace{\Psi([\![\fsharp]\!]) - \hat{\Psi}(\ldkakko \fsharp \rdkakko)}_{\clubsuit} + 
\underbrace{\hat{\Psi}(\ldkakko \fsharp \rdkakko) - \hat{\Psi}(\ldkakko \fhat \rdkakko)}_{\diamondsuit} + \hat{\Psi}(\ldkakko \fhat \rdkakko).
\end{align*}
Since the truncation operation $\ldkakko \cdot \rdkakko$ does not increase the $\|\cdot\|_\infty$-norm of two functions,
we can apply the inequality \eqref{eq:PunifGmBound} and Lemma \ref{lemm:RademacherBoundOfG}
also for $\ldkakko \fsharp \rdkakko$ to bound the term $\clubsuit$.
The term $\diamondsuit$ can be bounded as
\begin{align*}
\hat{\Psi}(\ldkakko \fsharp\rdkakko) - \hat{\Psi}(\ldkakko \fhat\rdkakko )
& \leq \frac{1}{n} \sum_{i=1}^n | \psi(y_i,\ldkakko \fsharp(x_i)\rdkakko) - \psi(y_i,\ldkakko \fhat(x_i)\rdkakko) |
\leq \frac{1}{n} \sum_{i=1}^n \rho | \ldkakko \fsharp(x_i) \rdkakko - \ldkakko \fhat(x_i)\rdkakko | \\
& \leq \rho \sqrt{\frac{1}{n} \sum_{i=1}^n ( \ldkakko \fsharp(x_i) \rdkakko -\ldkakko  \fhat(x_i) \rdkakko )^2 } = 
\rho \|\ldkakko \fsharp \rdkakko- \ldkakko \fhat \rdkakko \|_n
\leq \rho \| \fsharp -  \fhat  \|_n
\leq \rho \deltanone.
\end{align*}
Combining these inequalities, we obtain the assertion.
\end{proof}

\section{Additional numerical experiments}

This section gives additional numerical experiments for compressing the network.

%

\subsection{Compressing VGG-16 on ImageNet}
\label{sec:ImageNetAppendix}

Here, we also applied our method to compress a publicly available VGG-16 network \cite{simonyan2014very}  on the ImageNet dataset.
We apply our method to the ImageNet dataset \cite{deng2009imagenet}.
We used the ILSVRC2012 dataset of the ImageNet dataset, which 
consists of 1.3M training data and 50,000 validation data. 
Each image is annotated into one of 1,000 categories.
We applied our method to this network 
and compared it with existing 
 methods, 
namely
APoZ \cite{hu2016network},
SqueezeNet \cite{iandola2016squeezenet}, and ThiNet \cite{iccv2017ThiNet}.
All of them are applied to the same 
VGG-16 network.
For fair comparison, 
we followed the same experimental settings as \cite{iccv2017ThiNet}; the way of 
training data generation, data augmentation, performance evaluation schemes and so on.

The results are summarized in Table \ref{tab:ImageNet-VGG-results_mod}.
It summarizes 
the Top-1/Top-5 classification accuracies,
the number of parameters (\#Param),
and the float point operations
(FLOPs) to 
classify a single image.
Our method is indicated by
``{\Ours}-(type)."
We employed the simultaneous procedure for compression.
In {\Ours}-Conv, we applied 
our method only to the convolutional
layers (it is not applied to the
fully connected layers (FC)).
The size of compressed network $\fsharp$ was set to be the same as that of ThiNet-Conv.
{\Ours}-GAP is a method 
that replaces 
the FC layers of {\Ours}-Conv
with a global average pooling (GAP) layer
\cite{lin2013network,zhou2016learning}.
Here, we again set the number of channels in each layer of {\Ours}-GAP to be same as that of ThiNet-GAP.
We employed $\lambda_\ell = 10^{-6}\times \Tr[\Sigmahat_{(\ell)}]$ and $\theta = 0.5$ for our method.

We see that in both situations, out method outperforms ThiNet in terms of accuracy.
This shows effectiveness of our method while our method is supported by theories.

\begin{table}[htp]
\caption{Performance comparison on ImageNet dataset.
Our proposed method is compared with APoZ-2, 
and ThiNet. 
Our method is indicated as ``{\Ours}-(type)."
}
\label{tab:ImageNet-VGG-results_mod}
\label{tab:ImageNet-VGG-results}
\centering 
\begin{tabular}{|@{\hspace{0.05cm}}l@{\hspace{0.05cm}}|
@{\hspace{0.1cm}}ll@{\hspace{0.2cm}}r@{\hspace{0.15cm}}r|}
\hline 
Model & Top-1 & Top-5 & \# Param. & FLOPs\\
\hline 
\hline 
Original VGG \cite{simonyan2014very}  & 68.34\% & 88.44\% &  138.34M & 30.94B \\
APoZ-2  \cite{hu2016network} & 70.15\% & 89.69\% & 51.24M & 30.94B  \\
ThiNet-Conv  \cite{iccv2017ThiNet}&
69.80\% & 89.53\% & 131.44M & 9.58B \\
ThiNet-GAP \cite{iccv2017ThiNet}& 
67.34\% & 87.92\% & 8.32M & 9.34B \\
\hline 
\hline
{\bf {\Ours}-Conv}&
{\bf 70.418\%} &{\bf 90.094\%}  & 131.44M & 9.58B \\
{\bf {\Ours}-GAP}&
{\bf 67.540\%} & {\bf 88.270\%}  & 8.32M & 9.34B \\
\hline 
\end{tabular}
\end{table}

\end{document}